\documentclass[onefignum,onetabnum]{siamonline181217}

\usepackage[utf8x]{inputenc}
\usepackage[english]{babel}
\usepackage{microtype}
\usepackage{amsmath}
\usepackage{amsfonts}
\usepackage{amssymb}
\usepackage{bm} % write every math-letter boldface using \bm, including greek characters..
\usepackage{cleveref}

\usepackage{graphicx}
\usepackage{cite}
\usepackage{algorithmic}
\ifpdf
  \DeclareGraphicsExtensions{.eps,.pdf,.png,.jpg}
\else
  \DeclareGraphicsExtensions{.eps}
\fi

\usepackage{enumitem}
\setlist[enumerate]{leftmargin=.5in}
\setlist[itemize]{leftmargin=.5in}

\headers{Blind identification of stochastic block models from dynamical observations}{M. T. Schaub, S. Segarra, and J. Tsitsiklis}

\title{Blind identification of stochastic block models from dynamical observations\thanks{Submitted to the editors DATE.
\funding{This work was supported by the European Union's Horizon 2020 research and innovation programme under the Marie Sklodowska-Curie grant agreement No 702410 (M. Schaub).}}}

\author{Michael T. Schaub\thanks{Institute for Data, Systems and Society, Massachusetts Institute of Technology, USA \newline\indent and Department of Engineering Science, University of Oxford, UK 
    (\email{michael.schaub@eng.ox.ac.uk}).}
\and Santiago Segarra\thanks{Department of Electrical and Computer Engineering, Rice University, USA
  (\email{segarra@rice.edu}).}
  \and John N. Tsitsiklis\thanks{Laboratory of Information and Decision Systems, Massachusetts Institute of Technology, USA (\email{jnt@mit.edu}).}}

\usepackage{amsopn}
\DeclareMathOperator{\diag}{diag}

\newsiamremark{remark}{Remark}
\crefname{hypothesis}{Hypothesis}{Hypotheses}

\newsiamthm{problem}{Problem}
\newsiamthm{example}{Example}
\newsiamthm{assumption}{Assumption}

\newcommand{\nA}{\bm{L}}
\newcommand{\nEA}{\bm{\mathcal{L}}}
\newcommand{\EA}{\bm{\mathcal A}}
\newcommand{\ED}{\bm{\mathcal D}}

\newcommand{\elambda}{\mu}
\newcommand{\sCov}{\widehat{\bm{\Sigma}}(T)}
\newcommand{\Cov}{\bm{\Sigma}(T)}
\def\smid{\,|\,}

\usepackage{xcolor}

\ifpdf
\hypersetup{
  pdftitle={Blind identification of stochastic block models from dynamical observations},
  pdfauthor={M. T. Schaub, S. Segarra, and J. Tsitsiklis}
}
\fi

\begin{document}

\maketitle

% REQUIRED
\begin{abstract} 
We consider a blind identification problem in which we aim to recover a statistical model of a network without knowledge of the network's edges, but based solely on nodal observations of a certain process.   
More concretely, we focus on observations that consist of single snapshots taken from multiple trajectories of a diffusive process that evolves over the unknown network. 
We model the network as generated from an independent draw from a latent stochastic block model (SBM), and our goal is to infer both the partition of the nodes into blocks, as well as the parameters of this SBM.
We discuss some non-identifiability issues related to this problem and present simple spectral algorithms that provably solve the partition recovery and parameter estimation problems with high accuracy.
Our analysis relies on recent results in random matrix theory and covariance estimation, and associated  concentration inequalities.
We illustrate our results with several numerical experiments.
\end{abstract}

% REQUIRED

% REQUIRED

\section{Introduction}

We consider the problem of inferring a statistical model of a network with unknown adjacency matrix $\bm A \in \{0,1\}^{n \times n}$, based on snapshots of multiple trajectories of a dynamical process that evolves over the network.
Specifically, we are interested in the following setup (a more detailed description will be provided in~\Cref{sec:problem_setup}):
We observe $s$ independent samples $\bm x^{(1)}_T, \ldots,\bm x^{(s)}_T$ of a dynamical process at a fixed time $T$.
Each sample is generated by the discrete-time system
\begin{equation}\label{eq:general_setup}
    \bm x_{t+1}^{(i)} = f(\bm A) \bm x_t^{(i)}, \qquad i = 1,\ldots, s,
\end{equation}
initialized at some unknown $\bm x_0^{(i)}$, independently for each $i$.
Each $\bm x_0^{(i)}$ is modeled as a random vector with unit covariance matrix ($\mathbb{E}[\bm{x}_0^{(i)} [\bm{x}_0^{(i)}]^\top] = \bm{I}$) drawn from some sub-Gaussian distribution (e.g., a standard multivariate Gaussian).
Furthermore, we assume that the function $f$ describing the dynamical network process is known and corresponds to a class of diffusive processes.
Finally, $\bm A$ is assumed to be drawn independently according to a stochastic block model (SBM), so that $\bm A$ is the same, albeit unknown, for each sample $\bm x_T^{(i)}$.
Based on such data, our task is to identify the generative model underpinning the network:  we wish to infer the parameters and the latent node partition of the SBM, 
without observing the network edges, by relying solely  on the observations $\{\bm x_T^{(i)}\}_{i=1}^s$ at the nodes.

\subsection{Motivation}
In many applications we are confronted with the following scenario: we observe snapshots of the state of a system at particular times, and based on these observations we wish to infer the nature of the (dynamical) interactions between the entities that we observe.
Examples abound and can be found across different fields~\cite{timme2014revealing,ljung1998system,ljung2010perspectives}.

Consider, e.g., measuring some node activities such as the expression of opinions at $s$ different, sufficiently separated instances of time in a social network.
While the underlying social connections are typically not observable, the latent (unobserved) network will be essentially static, and may be assumed to have emerged according to some homophilic interactions describable by an SBM.
The measured opinion profiles will be shaped by both some (random) initial personal opinions of the individuals as well as certain network effects.
Similar examples may arise when measuring (structurally stable) biological systems responding to different stimuli.
For instance, we might have access to a set of measurements from certain regions of the brain and would like to infer the coupling between these regions.
As another example, we might be interested in inferring a network of genetic interactions from multiple snapshots of gene expression, e.g., as obtained from single cell RNA-sequencing~\cite{Saliba2014}. 
As taking these expression measurements will destroy the system under consideration, we are left with a few snapshots of data, from which we want to reconstruct (a model of) the underlying interactions.
Finally, we may consider the monitoring of ecological populations from abundance data~\cite{Gray2014}, where we often can only access snapshot information at specific instances of time, but are interested in knowing the network of interactions between the species.  

The above are canonical examples of problems where the system of interest is composed of a large number of different entities and can be conveniently conceptualized as a network.
In this abstraction, nodes represent the individual entities and the couplings between these nodes take place along the edges of the network. 
In many settings, we may have reason to believe that an observed dynamical process emerges from the interaction between nodes over such a network. While we may have some knowledge about the dynamics of each individual node, the edges of the network -- which enable the interactions between nodes -- are often unknown or only partially known. 
In such cases it is often desirable to learn the structure of the network based on observational data. 
This is a general class of network inference problems, which is a fundamental constituent of network analysis, but has received limited attention.

On the other hand, there are many situations where inferring exactly the couplings  in the network e.g., identifying all edges,  is out of reach, for a number of reasons.
First, we may not have access to sufficiently many samples. 
Second, the data and the network itself may be subject to some fluctuations with associated uncertainties.
Finally, we may only be able to partially observe the dynamical system,  so that the full model is non-identifiable.

In such cases, it is then reasonable to aim at an \emph{approximate} description of the system at hand, in terms of a statistical model that captures the coarse-grained features of the underlying network. This is precisely the type of problem addressed in this paper, for the special case where
the dynamics are of a linear, diffusive type, and the network model to be inferred is an SBM.

\subsection{Contributions}
We present a rigorous study of blind model identification for random networks.
Our main contributions are as follows:

First, we present general algorithmic ideas on how to solve a blind identification problem using spectral methods, and introduce simple algorithms to solve the concrete problems considered in this paper (see~\Cref{sec:algorithmic_ideas,sec:parameter_rec}, and~\Cref{alg,alg2}).

Second, we establish a concentration inequality about  the sample covariance matrix that is calculated on the basis of the observed samples (see~\Cref{T:main_result,T:main_result_gen_f}).

Finally, we discuss how this concentration result translates into explicit bounds for the recovery error of the latent partitions (see~\Cref{T:result_misclassification}), and provide bounds on the estimation errors for the parameters of the SBM, for the specific example of a planted partition model (see~\Cref{T:parameter_estimation_result}).

\subsection{Related work}
Many different network inference notions have featured in the literature.
For instance, the estimation of ``functional'' couplings based on statistical association measures such as correlation or mutual information, is an analysis often encountered in the context of biological systems~\cite{honey2009predicting,overbeek1999use,Greicius2009,Smith2002}.
Another type of network inference is causal inference, as discussed in, e.g.,~\cite{Bolstad2011,Pearl2009,Spirtes2016}, in which one aims to infer directed graphs that encode direct causal influences between the entities.
In contrast, the type of network inference we consider here is what may be called ``topological'' inference: given a system of dynamical units, we want to infer their direct ``physical'' interactions.
Some recent surveys include~\cite{timme2014revealing,Wang2016,mateos_2018_connecting,giannakis2018topology}. 
More specifically, for an evolving dynamical process we might want to infer the underlying adjacency matrix. 
This may be viewed as a system identification problem, as considered in Control Theory~\cite{ljung1998system,ljung2010perspectives}. 
However, in the system identification literature, one can usually control the inputs to the system and observe the full system response.
This is different from our setting, where there is no control, and we only observe a limited number of snapshots.

Problems of network identification that are somewhat similar to ours have received wide interest in the literature recently.
In many cases, formulations based on the optimization of a sparsity-promoting cost function are used~\cite{Julius2009,Candes2008,Hayden2016}, or formulations based on sparse dictionary learning~\cite{Yuan2011,Brunton2016}.
A number of methods also employ spectral characteristics of the observed outputs to identify (aspects of) the network topology~\cite{Hayden2017,Materassi2012,Mauroy2017}.
In particular, the case of diffusive dynamics, similar to ours, is considered in~\cite{Shahrampour2013} and~\cite{Shahrampour2015}, which aim at inferring the network topology based on the (cross-) power spectral density of a consensus process driven by noise, and a node knockout strategy. 
Methods combining spectral identification with optimization techniques are studied in~\cite{Segarra2017,wai2016active, segarra_2017_network, shafipour_2017_network}. 

In all of the above works, however, the focus is on inferring the \emph{exact} network topology, i.e., the complete specification of the network in terms of its adjacency or Laplacian matrices.
In contrast, our goal is not to infer a specific network coupling, but rather to identify a statistical model thereof.
Stated differently, we wish to infer a distribution over networks that is compatible with the available observations.
This scenario is particularly relevant if our ability to reconstruct the network is limited by the amount of samples that we can obtain.
Closest to our approach is the conference paper~\cite{schaub_2019_spectral}, in which a related observational model is discussed, with a \emph{new adjacency matrix} drawn for each sample in~\eqref{eq:general_setup}, and which provides results about the asymptotic consistency of partition recovery.
In contrast, we consider the case of a constant but unknown adjacency matrix, provide non-asymptotic results in the number of samples and, in addition, we address the problem of parameter estimation.

Similar to this paper, the works~\cite{Wai2018,Wai2018a} study the `blind solution' of a ratio-cut problem (a common heuristic for partitioning graphs into dense, assortative clusters) based on observations of the output of a low-pass graph filter, excited by a low-rank excitation signal.
They provide results on how well the ratio cut-problem can be approximated on arbitrary, fixed graphs, under the assumption of a low-rank excitation structure.
In contrast, here we aim to infer an explicit generative graph model and can thus provide probabilistic performance guarantees about the recovery of the true planted partition.
The work by Hoffmann et al.~\cite{Hoffmann2018} and Peixoto~\cite{Peixoto2019} discuss hierarchical Bayesian models and inference schemes for related problems, but do not provide a theoretical analysis or statistical guarantees.
Other related problems include the study of multi-reference alignment and synchronization problems (see, e.g.,~\cite{Perry2019,Bandeira2017}), or the study of state compression in Markov chains~\cite{Zhang2018}.
See also~\cite{Nitzan2017} for an example in which network reconstruction is attempted based on dynamics, but without knowledge of temporal information.

Finally, the mathematical machinery that we use has connections to ideas from spectral methods for clustering on graphs and SBM recovery~\cite{vonLuxburg2007tutorial, rohe2011spectral, qin2013regularized,Abbe2018}, and is related to the statistics literature on so-called ``spiked'' random matrices, with a low-rank signal component.
In particular, our development relies on recent results regarding the efficiency of estimating low-rank covariance matrices~\cite{bunea2015sample}.

\subsection{Outline}
The rest of the paper is structured as follows.
After briefly reviewing relevant preliminaries and notation in Section~\ref{sec:preliminaries}, in Section~\ref{sec:problem_setup} we formally introduce the  identification problem that we will study.
We then present our main results and algorithmic ideas on recovering the latent \emph{partitions} of the SBM in Section~\ref{sec:algorithmic_ideas}.
Building on those insights, in~\Cref{sec:parameter_rec} we discuss procedures for learning the \emph{parameters} of the SBM.
\Cref{sec:theoretical_analysis} contains the proofs of our main technical results, and in~\Cref{sec:numerical_exp} we present computational experiments to numerically assess the performance of our methods.
We conclude with a brief discussion, and outline possible avenues for future work.

\section{Preliminaries}\label{sec:preliminaries}
\subsection{Networks} 
An undirected network $\mathcal G$ consists of a node set $\mathcal V$ of known cardinality $n$, and an edge set $\mathcal E$ of unordered pairs of elements of $\mathcal V$.
The edges can be conveniently represented as entries of a symmetric adjacency matrix $\bm A\in\{0,1\}^{n\times n}$, such that $A_{ij}=A_{ji} = 1$ if and only if $(i,j)\in\mathcal E$. 
We define the diagonal degree matrix $\bm{D} = \diag{(\bm{A}\bm 1)}$, where $\bm 1$ denotes the vector of all ones and $\text{diag}(\bm x)$ denotes the diagonal matrix whose entries are given by the components of the vector $\bm x$.
We define the normalized adjacency matrix as $\nA=\bm{D}^{-1/2}\bm{A}\bm{D}^{-1/2}$.
Note that $\nA$ is a shifted version of the normalized Laplacian matrix of the graph, and is sometimes referred to in the literature simply as the ``normalized Laplacian''~\cite{rohe2011spectral,qin2013regularized}.

\subsection{The stochastic block model}
The SBM is a latent variable model that defines a probability measure over the set of undirected networks of fixed size $n$, represented by an adjacency matrix $\bm A \in \{0,1\}^{n\times n}$.
In an SBM, the network is assumed to be partitioned into $k$ nonempty groups of nodes, so that each node $i$ is endowed with a latent group label $g_i \in\{1,\ldots,k\}$.
Given these latent group labels, each link $A_{ij}$ is an independent Bernoulli random variable that takes value $1$ with probability $\Omega_{g_i,g_j}$ and value $0$ otherwise; that is, the probability of a link between two nodes depends only on their group labels:\footnote{For simplicity, our model allows self-loops.}
\begin{equation*}
     A_{ij} \sim \text{Ber}(\Omega_{g_i,g_j}).
\end{equation*}

To compactly describe the model, we collect link probabilities between the different groups in a symmetric affinity matrix $\bm \Omega = [\Omega_{ij}] \in [0,1]^{k \times k}$.
Furthermore, we define a partition indicator matrix $\bm G \in \{0,1\}^{n\times k}$, with entries $G_{ij} = 1$ if node $i$ is in group $j$, and $G_{ij}=0$ otherwise.
Based on these definitions, we can write the expected adjacency matrix under the SBM,  given the partition and the group labels as\footnote{Note the slight abuse of notation here with respect to the conditional expectation. We treat the partition indicator matrix $\bm G$ as a degenerate random variable with a deterministic, but unknown value.}
\begin{equation}\label{eq:low_rank}
    \bm{\mathcal{A}} := \mathbb{E}[\bm A \smid \bm G] = \bm G \bm \Omega \bm G^\top.
\end{equation}
As $\bm{\mathcal{A}}$ encodes all the parameters of an SBM, we will in the sequel refer to such a stochastic block model as $\mathcal M(\bm{\mathcal{A}})$.
Throughout the rest of the paper, and in order to avoid degenerate cases, 
we will always assume that $\bm{\mathcal{A}}$ has rank $k$.
In particular, this excludes the case where two rows in $\bm{\mathcal{A}}$ are  co-linear, i.e., the case where the connection probabilities of two groups  are identical or a linear scaling of each other.\footnote{Our main results rely on the analysis of a $k$-dimensional dominant subspace in the observed data, which is induced by the structure of $\bm{\mathcal{A}}$. If $\bm{\mathcal{A}}$ had rank $p<k$, then our arguments would need to be adjusted accordingly.}

The SBM leads to a network whose expected adjacency matrix has a low-rank block structure (rank $k$).
Due to the special structure induced by \Cref{eq:low_rank}, there exists a simple relationship between the partition indicator matrix $\bm G$ and the eigenvectors of the normalized version of the expected adjacency matrix $\nEA$, defined by
\begin{equation}
    \nEA := \mathbb E[\bm{D}]^{-1/2}\mathbb{E}[\bm A]\mathbb{E}[\bm{D}]^{-1/2} = \text{diag}(\bm{\mathcal A}\bm 1)^{-1/2} \bm{\mathcal{A}} \text{ diag}(\bm{\mathcal A}\bm 1)^{-1/2}.
\end{equation}
The relationship between $\bm G$ and the eigenvectors of $\nEA$ is given in Lemma~\ref{T:eigenvectors_and_partitions} below, and proved in the Appendix. 
We denote the eigenvalues of $\nEA$ by $\elambda_i$, and assume that they are sorted in order of decreasing magnitude, so that $|\elambda_1| \ge |\elambda_2| \ge \cdots \ge |\elambda_k| > |\elambda_{k+1}| = \cdots = |\elambda_n| = 0$. 
\begin{lemma}\label{T:eigenvectors_and_partitions}
    Consider the matrix $\bm V = [\bm v_1, \ldots, \bm v_k]$ of eigenvectors of $\nEA$ corresponding to the $k$ nonzero eigenvalues. Then,
    there exists a $k\times k$ orthogonal matrix $\bm{\mathcal U}$ such that
    \begin{equation}
    \bm V = \bm G (\bm G^\top \bm G)^{-1/2} \bm{\mathcal U}.
    \end{equation}
	Moreover, if we define the row-normalized eigenvector matrix $\widetilde{\bm V} = \diag(\bm V\bm V^\top)^{-1/2}\bm V$, then
    \begin{equation}
        \widetilde{\bm V} = \bm G \bm{\mathcal U}.
    \end{equation}
\end{lemma}

Given an SBM with an expected adjacency matrix $\bm{\mathcal{A}}$ as described above, let us denote the expected degree of node $i$ by $\delta_i := \mathbb{E}[\bm{D}_{ii}] = [\bm{\mathcal{A}}\bm 1]_i$, and let $\delta_\text{min}$ be the smallest expected degree.
We will restrict attention to SBMs with a sufficiently large $\delta_\text{min}$.

\begin{definition}[SBMs with concentration property]\label{def:SBMSconcentrate}
Given some $\epsilon>0$, we define $\mathcal M_n(\epsilon)$ as the class of SBMs $\mathcal{M}\left(\bm{\mathcal{A}}\right)$ with $n$ nodes, for which
\begin{equation*}
    \delta_\text{min}  > 27\ln\left(\frac{4n}{\epsilon} \right).
\end{equation*}
\end{definition}
The condition on the minimal degree essentially restricts $\mathcal M_n(\epsilon)$ to correspond to models such that all nodes that belong to the same group belong to the same connected component, with high probability.\footnote{{Note that for the special case of a single group, we are dealing with an Erd\H{o}s-R\'enyi random graph, and it is well known that if $\delta_\text{min} > c {\ln(n)}$, with $c>1$, the random graph is connected with high probability, as $n\rightarrow \infty$.}}
This is what will allow us to obtain reliable group membership  information from the link patterns. 
Indeed, without such an assumption, and even if we knew the adjacency matrix $\bm A$, the exact recovery (with high probability) of the partition would be impossible~\cite{Abbe2018}.

For the above class of models the following concentration result can be established.
\begin{lemma}[Concentration for normalized adjacency matrix]\label{T:concentration_A}
    Fix some $\epsilon >0$.
    Consider a normalized adjacency matrix $\bm L := \bm D^{-1/2}\bm A\bm D^{-1/2}$, computed from an adjacency matrix $\bm A$ generated according to an SBM model $\mathcal M(\bm{\mathcal{A}})$.
    If $\mathcal M(\bm{\mathcal A}) \in \mathcal M_n(\epsilon)$, then the following bound holds with probability at least $1-\epsilon$:
    \begin{equation*}
        \left \|\nA -\nEA\right \| \le 3 \sqrt{\frac{3 \ln(4n/\epsilon)}{\delta_\text{min}}},
    \end{equation*}
    where we recall that ${\nEA := \mathbb E[\bm{D}]^{-1/2}\mathbb{E}[\bm A]\mathbb{E}[\bm{D}]^{-1/2}}$.
\end{lemma}
\Cref{T:concentration_A} follows from a result by Chung and Radcliffe~\cite{chung2011spectra}. 
The proof, which parallels~\cite{chung2011spectra}, is given in the Appendix for completeness.
In the following, we will be primarily interested in  the case where $\delta_\text{min} = \omega\left(\ln(n/\epsilon)\right)$. In this case, when $n$ is large, the condition 
in~\Cref{def:SBMSconcentrate} will be satisfied.\footnote{More precisely, the condition  $\delta_\text{min} = \omega\left(\ln(n/\epsilon)\right)$ means that $\delta_\text{min}/{\ln(n/\epsilon)} \rightarrow \infty$, as $n\rightarrow \infty$.}
\Cref{T:concentration_A} then guarantees a tight spectral concentration for large networks, i.e., that  $\left \|\nA -\nEA\right \|$ will be small, with high probability.

\subsection{Sub-Gaussian random vectors}
For convenience, we recall here the definition of sub-Gaussian random variables and vectors.
\begin{definition}[Sub-Gaussian random variables and vectors]
A zero-mean random variable $x\in \mathbb{R}$ is \emph{sub-Gaussian} if there exists a constant $\sigma > 0$ such that $\mathbb{E}[\exp(tx)] \le \exp(t^2\sigma^2/2)$, for all $t\in \mathbb{R}$.
A zero-mean random vector $\bm x \in \mathbb{R}^n$ is sub-Gaussian if for any non-random $\bm u \in \mathbb{R}^n$, the random variable $z = \bm u^\top \bm x$ is sub-Gaussian.
\end{definition}

Note that any Gaussian random variable or vector is sub-Gaussian.
We also define the sub-Gaussian norm as follows.
\begin{definition}[Sub-Gaussian norm~\cite{bunea2015sample}]
    The sub-Gaussian norm of a sub-Gaussian random variable $x$ is $\|x\|_{\Psi_2} := \sup_{k\ge 1}k^{-1/2}(\mathbb{E}|x|^k)^{1/k}$.
    For a sub-Gaussian random vector $\bm x$, the sub-Gaussian norm $\|\bm x\|_{\Psi_2}$ is defined as ${\|\bm x\|_{\Psi_2} := \sup_{\bm u \in \mathbb{R}^n \backslash \{0\} }\|\bm u^\top \bm x\|_{\Psi_2}/ \|\bm u\|}$.
\end{definition}

%%%%%%%%%%%%%%%%%%%%%%%%%%%%%%%%%%%%%%%%%%%%%%%%
\section{The blind identification problem}\label{sec:problem_setup}
%%%%%%%%%%%%%%%%%%%%%%%%%%%%%%%%%%%%%%%%%%%%%%%%
In this section, we introduce the specific inference problem that we will study.
We first provide a brief discussion on the class of dynamical models we consider, and then describe the precise problem setup.

\subsection{Dynamical model}
We consider the dynamical model described in~\cref{eq:general_setup}, where the mapping $f$ is of the form
\begin{equation}\label{eq:function_class}
    f: \bm A \mapsto \sum_{\ell=1}^{R} \alpha_{\ell} \nA^\ell,
\end{equation}
for some known constants $\alpha_{\ell}$, and a known degree $R$,
where $\nA =\bm D^{-1/2}\bm A\bm D^{-1/2}$ is the (unknown) normalized adjacency matrix of the (unobserved) underlying graph. 
We assume that the system matrix satisfies $\|f(\bm A)\| \le 1$, and corresponds to a (marginally) stable dynamical system.

The above mapping $f$ can be considered to describe some local interactions between the nodes, in that it allows for interactions between nodes that are up to $R$ steps away in the graph.
The mapping \cref{eq:function_class} $f$ may alternatively be considered from the perspective of graph signal processing~\cite{Ortega2018,mateos_2018_connecting} as a graph filter of order $R$, that acts (instantaneously) on the state variables of the dynamical system \cref{eq:general_setup} at each time step.

To make our discussion more concrete we will in the following concentrate on the dynamical system
\begin{equation}\label{eq:observation_model_simple}
        \bm x_{t+1}  = \nA\bm{x}_{t},
\end{equation}
as our main example.
This discrete-time system, and by extension the dynamics in~\eqref{eq:observation_model}, is closely related to diffusion processes on networks~\cite{masuda2017random, segarra_2017_filters}. 
Indeed, $\nA$ is related via a similarity transformation to the (column-stochastic) transition matrix $\bm{A}\bm{D}^{-1}$ (or diffusion kernel) of an unbiased random walk on the network.
Such random walk dynamics have been used to model a variety of processes, such as the spreading of rumors and opinions, and have found applications in other areas of network analysis such as ranking nodes or detecting communities~\cite{masuda2017random}.
Another interpretation of the model is in terms of (discrete-time) consensus processes, which are distributed averaging processes that have been considered extensively in the Control Theory literature~\cite{tsitsiklis1984problems,jadbabaie2003coordination,olfati2007consensus}.

It is well known that $\|\bm A\bm D^{-1}\| = 1$  and accordingly $\|\bm L \| = 1$, since $\bm L$ and $\bm A \bm D^{-1}$ are related via a similarity transform.
Hence, the assumption $\|f(\bm A)\|\le 1$ is fulfilled by our example dynamics. 
We will make use of this fact in our calculations later.

\begin{remark}
    By allowing for a large enough $R$ we can view the observation $\bm x_T$ at time $T$ as a signal $\bm x_0$ filtered by a graph filter with mapping $f'(\bm A)$, where $f'(\bm A)$ is obtained by composing $T$ copies of the original function $f$. However, we refrain from considering this `static' filter perspective here, as this would disguise a trade-off between the sampling time and the number of samples required for our identification problem, which will become apparent in the following sections.
\end{remark}

\subsection{Problem description}
Let us provide a detailed breakdown of the observation model we consider.
\begin{enumerate}
    \item We have a discrete-time dynamical system with state vector $\bm x_t \in \mathbb{R}^n$, which evolves according to
    \begin{equation}\label{eq:observation_model}
        \bm x_{t+1}  = f(\bm A)\bm{x}_{t},
    \end{equation}
    where $f$ is of the form~\cref{eq:function_class}.
    
\item We assume that the normalized adjacency matrix $\nA$ that appears in $f(\bm A)$ is generated by first drawing a graph from an SBM with $k$ blocks, (unknown) partition indicator matrix $\bm G$, and some (unknown) affinity matrix $\bm \Omega$, and then normalizing its adjacency matrix $\bm A$ to obtain $\nA$ and $f(\bm A)$.

    \item We obtain $s$ samples $\bm x_T^{(i)}$, $i=1,\ldots, s$, at a known time $T$, associated to $s$ different trajectories.
  
    \item The (unobserved) initial vectors $\bm{x}_0^{(i)}$ for the different samples are i.i.d., drawn from a  zero-mean sub-Gaussian vector distribution, where the components of the vector are uncorrelated, i.e., $\mathbb{E}[\bm x_0^{(i)} [\bm x_0^{(i)}]^\top] = \bm I$.
  
\end{enumerate}

Under this setup, we consider the following two problems, which we refer to as the \emph{partition recovery} and  \emph{parameter estimation} problems. 
\begin{problem}{(SBM partition recovery from dynamics).}
    Given $s$ independent samples, each one drawn at a fixed time $T$ according to the above described model, recover the latent partition of the nodes into blocks.  
\end{problem}

\begin{problem}{(SBM parameter estimation from dynamics).}
Given $s$ independent samples, each one drawn at a fixed time $T$ according to the above described model, recover the affinity matrix $\bm \Omega$ describing the link probabilities in the underlying SBM.
\end{problem}

Note that, as the assignment of labels to groups is arbitrary, we can only hope to recover the block structure and the affinity matrix $\bm \Omega$ up to a permutation of the group labels.
Similarly, as will become clear in the following sections, the parameters of the SBM will in general not be identifiable, unless some additional assumptions are made on the underlying model or some additional side information is available on the structure of the network.

Moreover, it should be apparent that the difficulty of the above problems will depend on the specific realization of the initial conditions $\bm x_0^{(i)}$ and of the adjacency matrix $\bm A$. 
Thus, we cannot expect accurate inference for every possible realization.
Instead we look for algorithms and results that, under certain assumptions, guarantee accurate inference with high probability.

We make the following technical assumption about the data generation process. 
\begin{assumption}[Bounded moments assumption]\label{assumption2}
    We assume that there exists an absolute constant  $c_0>0$ such that $\mathbb{E}[(\bm u^\top\bm x_0^{(i)})^2]\ge c_0 \|\bm u^\top \bm x_0^{(i)}\|^2_{\Psi_2}$, for all $\bm u$.
\end{assumption}
As discussed in~\cite{bunea2015sample}, Assumption~\ref{assumption2} effectively bounds all moments of $\bm x_0^{(i)}$ as a function of the second moment of $\bm x_0^{(i)}$.
As a special case, this assumption is satisfied by a standard Gaussian vector, with $c_0 = \pi/2$~\cite{bunea2015sample}.
\Cref{assumption2} will enable us to invoke (via~\Cref{T:boundSampleCovarianceToCovarianceBunea}) the results of~\cite{bunea2015sample} on the estimation of low-rank covariance matrices, thereby facilitating our analysis.

\begin{remark}
If the number of available samples becomes large ($s\rightarrow \infty$), we could in principle try to infer the exact adjacency matrix $\bm A$. 
However, we are interested in a situation where $s$ is comparatively small ($s \ll n$), and our goal is to infer the SBM model that generated $\bm A$ (and thus $\nA$), that is, the latent partition and the affinity matrix $\bm \Omega$.
\end{remark}

%%%%%%%%%%%%%%%%%%%%%%%%%%%%%%%%%%
\section{Partition recovery}\label{sec:algorithmic_ideas}
%%%%%%%%%%%%%%%%%%%%%%%%%%%%%%%%%%
In this section we first present our main results and an algorithm for solving the partition recovery problem (Problem 1) and then provide a broader discussion of the underlying algorithmic ideas.

\subsection{Main results}
Our algorithm for partition recovery relies on spectral clustering, and is displayed as \Cref{alg}.
\begin{algorithm}[t]
    \caption{\vspace{1pt}Partition recovery}%
    \label{alg}
    \begin{algorithmic}[1]
    \item[]{\textbf{Input:} Samples $\{\bm x_T^{(i)}\}_{i=1}^s$, Number of groups $k$}
    \item[]{\textbf{Output:} Partition $\mathcal P$}
  \STATE{Compute sample covariance $\sCov := \frac{1}{s}\sum_{i=1}^s [\bm x_T^{(i)} - \bm{\bar x}_T][\bm x_T^{(i)}- \bm{\bar x}_T]^\top$}
  \STATE{Compute top $k$ eigenpairs: $\{(\lambda_i,\bm w_i)\}_{i=1}^k \leftarrow \text{eig}(\sCov)$, where $|\lambda_1| \ge |\lambda_2| \ge \cdots \geq |\lambda_k|$}
  \STATE{Form the matrix $\bm W = [\bm w_1,\ldots,\bm w_k]\in \mathbb{R}^{n\times k}$ and normalize its rows to have unit norm: $\widetilde{\bm W} = \diag(\bm W \bm W^\top)^{-1/2}\bm W$}
  \STATE{Perform $k$-means clustering on the rows of $\widetilde{\bm W}$: $\mathcal P \leftarrow \text{$k$-means}(\widetilde{\bm W}^\top)$}
  \end{algorithmic}
\end{algorithm}
It builds upon the idea that the \emph{sample covariance matrix} $\sCov$ at time $T$, defined by
\begin{equation}\label{eq:sample_cov}
    \sCov := \frac{1}{s}\sum_{i=1}^s [\bm x_T^{(i)} - \bm{\bar x}_T][\bm x_T^{(i)}- \bm{\bar x}_T]^\top,
\end{equation}
where $\bm{\bar x}_T = \frac{1}{s}\sum_{i=1}^s \bm x_T^{(i)}$ is the sample mean, can serve as a proxy for what we call the \emph{ensemble covariance matrix} 
\begin{equation}\label{eq:ensemble}
\bm S(T):= 
f(\EA)^{2T}.
\end{equation}
Because the initial vector has identity covariance, it is seen that $\bm S(T)$ would be the covariance matrix of the vector $\bm x_T^{(i)}$ if the  dynamics were evolving according to the matrix $f(\EA)$, i.e., if we were to replace the adjacency matrix of the sampled network with its expected value.
Intuitively, the matrix $f(\EA)$ encodes the partition structure of the SBM, and since $\bm S(T)$ is simply a matrix power of $f(\EA)$, it has the same dominant eigenspace.
Being able to compute a close proxy for $\bm S(T)$ enables us to gain information about the partition structure. 
Thus, as long as the sample covariance matrix $\sCov$ is close to the ensemble covariance matrix $\bm S(T)$, we can expect good recovery performance. 

To gain some first insight, Theorem~\ref{T:main_result} below addresses the relation between these two covariance matrices, for our example dynamics $f(\bm A) = \nA$.
\begin{theorem}\label{T:main_result}
    Let $\nA$ be the normalized adjacency matrix constructed from a graph drawn from an SBM $\mathcal M(\bm{\mathcal{A}}) \in \mathcal M_n(\epsilon)$ with $n$ nodes and minimal degree $\delta_\text{min}$.
    Let $\bm x_0^{(i)}$ $i=1,\ldots,s$, be i.i.d. zero-mean sub-Gaussian vectors that satisfy~\Cref{assumption2}, with covariance matrix $\bm I$.
    
    Then, for the dynamics of the form \cref{eq:observation_model_simple}, and with probability at least $1-5s^{-1}-2\epsilon$, we have:
    \begin{equation}\label{E:rate_main_result}
        \left\|\sCov - \bm S(T) \right \| \le 2T M + B(T),
    \end{equation}
    where  $\sCov$ is defined in \eqref{eq:sample_cov}, $\bm S(T)$ is defined in \eqref{eq:ensemble},
    \begin{equation*}
        M := 3 \sqrt{\frac{3 \ln(4n/\epsilon)}{\delta_\text{min}}}, \qquad B(T) := C\left (1\!+\!\sum_{i=2}^k\min\{1, |\elambda_i| + M\}^{2T} +(n\!-\!k)M^{2T} \right) \sqrt{\frac{\ln s}{s}},
    \end{equation*}
    $C$ is some absolute constant, and where $\elambda_i$ are the eigenvalues of $\nEA$ ordered by decreasing magnitude.
\end{theorem}
The proof of~\Cref{T:main_result} is given in~\Cref{sec:concentration_of_covariance}.\footnote{The probability $1-5s^{-1}-2\epsilon$ in the statement of the theorem can be improved to
$1-5s^{-1}-\epsilon$ by a more careful use of the union bound in the proof. The same comment applies to \Cref{T:main_result_gen_f,T:result_misclassification} below.}

For some interpretation, note that the term $2TM$ in~\eqref{E:rate_main_result} is independent of the number of samples.
As we will see below, this term arises because the graph on which the dynamics takes place is only a proxy for the expected graph under the SBM, and is drawn only once.
We may think of this error term as a `baseline' insurmountable error that we will incur no matter how many samples we collect.
In light of the condition on $\delta_{\rm min}$ in Definition \ref{def:SBMSconcentrate}, we have $M<1$, but nevertheless the term $2TM$ in the error bound increases (linearly) with time.
Hence, the later we make our observations (larger $T$), the worse the first term in our bound  becomes.
From the perspective of our model~\eqref{eq:observation_model}, this is intuitive: since $\lambda_i(\nA) \le 1$ for all $i$, the relevant signal that we observe will be attenuated over time, making it more difficult to extract information about the embedded partition from our observations when $T$ is large.

In contrast to the first term, the second term $B(T)$ in \eqref{E:rate_main_result} depends both on the time $T$ and the number of samples $s$.
The sample dependence is related to the convergence rate of the sample covariance matrix $\sCov$ to the true covariance.
The time-dependence of $B(T)$ is again related to $M$, but in a non-linear fashion. 
Since $M < 1$, the error term $B(T)$ decreases with time.

The above discussion reveals a trade-off. If $T$ is large, then the first term, $2TM$, is large and dominates. If $T$ is small, then the second term, $B(T)$, can be large, especially if the number of samples, $s$, is moderate. In general, an intermediate value of $T$ will be yielding a most favorable upper bound.

It is also insightful to consider the behavior of our bound for both the asymptotic regime in which we have a large number of samples $s\rightarrow \infty$ (even though this is not the setting of interest to us), as well as in the limit of large networks ($n\rightarrow \infty$).
If we consider the case $s \rightarrow \infty$, and in the regime where $\delta_\text{min} = \omega(\ln(4n/\epsilon))$, then the error term is optimal, exactly in line with the requirements of ``exact recovery'' of the partition for an SBM where the network (the matrix $\bm A$) is observed~\cite{Abbe2018}.
Similarly, assuming a sufficiently large $\delta_\text{min}$ as $n\rightarrow\infty$, the dependency on the number of samples is aligned with the optimal rate for the estimation of low-rank covariance matrices~\cite{bunea2015sample}.

For the general case, where $f$ is given by~\cref{eq:function_class}, an analogous result can be established.
\begin{theorem}\label{T:main_result_gen_f}
    Let $\nA$ be the normalized adjacency matrix constructed from a graph drawn from an SBM $\mathcal M(\bm{\mathcal{A}}) \in \mathcal M_n(\epsilon)$ with $n$ nodes and minimal degree $\delta_\text{min}$.
    Let $\bm x_0^{(i)}$ $i=1,\ldots,s$, be i.i.d. zero-mean sub-Gaussian vectors that satisfy~\Cref{assumption2}, with covariance matrix $\bm I$.
    
    Then, for the dynamics of the form \cref{eq:observation_model}, and with probability at least $1-5s^{-1}-2\epsilon$, we have:
    \begin{equation}\label{E:rate_main_result_gen_f}
        \left\|\sCov - \bm S(T) \right \| \le 2T M_f + B_f(T),
    \end{equation}
    where  $\sCov$ is defined in \eqref{eq:sample_cov}, 
    \begin{align*}
        M_f &:= 3 \sqrt{\frac{3 \ln(4n/\epsilon)}{\delta_\text{min}}} \cdot \left(\sum_{\ell=1}^R \ell \alpha_\ell\right), \\
        B_f(T) & := C\left (1\!+\!\sum_{i=2}^k\min\{1,|\zeta_i|\!+\!M_f\}^{2T} +(n\!-\!k)M_f^{2T} \right) \sqrt{\frac{\ln s}{s}},
    \end{align*}
    $C$ is some absolute constant, and where $\zeta_i$ are the eigenvalues of $f(\EA)$ ordered by decreasing magnitude.
\end{theorem}
The proof of this result is given in \Cref{sec:proof_main_gen_f}.

Based on the concentration result in Theorem~\ref{T:main_result_gen_f}, we can establish the following guarantee for the partition recovery problem.
This result is phrased in terms of the misclassification rate $q:=|\mathbb{M}|/n$, where $\mathbb{M}$ is the set of ``misclassified nodes'' formally defined in ~\Cref{sec:implications_inference}.\footnote{Importantly, this definition of misclassified nodes assumes that we can solve the $k$-means procedure exactly.
In practice, however, only a $(1+\epsilon_0)$-approximation may be obtained efficiently~\cite{Kumar2004}.}
\begin{theorem}\label{T:result_misclassification}
    Let $\xi_1 \ge \cdots \ge \xi_k$ be the top $k$ eigenvalues of the ensemble covariance matrix $\bm S(T) = f(\EA)^{2T}$, which we assume to be positive.
    Let $\bm W$ and $\bm V$ denote the matrices formed by the top $k$ eigenvectors of $\sCov$ and $\bm S(T)$, respectively, and denote their row-normalized versions by $\widetilde{\bm W}$ and $\widetilde{\bm V}$.
    Define $\tau$ to be the minimum 2-norm of any row of $\bm{W}$ and $\bm{V}$.
    Then, with probability $1-5s^{-1}-2\epsilon$ the following bound for the misclassification rate $q$ holds:
    \begin{equation}
        q \le \frac{64k}{n\tau^2\xi_k^2}(2TM_f + B_f(T))^2,
    \end{equation}
    where $B_f(T)$ is as defined in~\Cref{T:main_result_gen_f}.
\end{theorem}

The proof is given in~\Cref{sec:proof_p_recovery}.
\Cref{T:main_result,T:main_result_gen_f,T:result_misclassification} provide performance guarantees that show the interplay between the number of groups $k$, the network size $n$, and the spectral properties of the network, as encapsulated in $\mu_k$ and $\|\sCov - \bm S(T)\|$.
Before continuing, let us add  a few remarks on some aspects of these results.

\begin{enumerate}
\item The time $T$ at which the samples are obtained is not a required input to the algorithm (cf. Algorithm~\ref{alg}).
However, as is apparent from~\Cref{T:main_result,T:main_result_gen_f,T:result_misclassification}, the partition recovery performance is influenced by the sampling time.
In particular, since the eigenvalues $\xi_k$ of the ensemble covariance matrix $f(\EA)^{2T}$ can only decrease with time, the bound we can obtain for the misclassification rate is weaker in comparison to the concentration result for the covariance matrix (\Cref{T:main_result,T:main_result_gen_f}).

\item In  Algorithm~\ref{alg}, we require the number of blocks $k$ to be an input to our algorithm.
In principle, we can try to infer this parameter from the data, e.g., by looking at some appropriate statistics of the eigenvalue gaps; 
see for example~\cite{xiang2008spectral,vonLuxburg2007tutorial,sanguinetti2005automatic}, or the discussion in~\cite{bunea2015sample} on the detection of eigenvalue jumps in covariance matrix estimation.
However, we will not carry out this line of analysis in this paper.

\item The same approach can be used for other types of (linear) operators derived from a low-rank statistical network model such as the SBM.
For instance one may consider inference for other latent graph models, such as degree-corrected SBMs~\cite{dasgupta2004spectral,Karrer2011}.
\end{enumerate}

\begin{remark}
    Note that in~\Cref{T:result_misclassification}, we assume that the first $k$ eigenvalues of $\bm S(T)$ are positive.
    Under the premise that we are dealing with a $k$-group SBM $\bm{\mathcal{A}}$ with rank $k$, this will be generically true. 
    However, if $\text{rank}(\bm{\mathcal{A}}) = p < k$, then the above Theorem would need to be adjusted, as only the first $p$ eigenvectors would carry information about the latent block structure, leading to possible non-identifiability. 
    The case $\text{rank}(\bm{\mathcal{A}}) < k$ could arise, for example,  if we consider a block model with equal-sized groups, in which two rows of the affinity matrix $\bm \Omega$ are co-linear.
    For a related discussion on some of these issues see, e.g.,~\cite{gulikers2017spectral}.
    Less trivial examples are also possible, e.g., if the rank of $\bm{\mathcal{A}}$ is $k$ but $\text{rank}(f(\bm{\mathcal{A}})) <k$.
\end{remark}

\subsection{Algorithm and proof ideas}
We now provide a more technical discussion of the proposed algorithm for partition recovery from a high-level perspective.
The proofs in~\Cref{sec:theoretical_analysis} will make these ideas mathematically precise.
To simplify our discussion below, we will describe the algorithmic ideas in terms of the simple dynamics $f(\bm A) = \bm L$.
The general case is completely analogous.

The intuition underpinning our algorithm is that while we cannot observe the edges of the network, the network structure affects the second moments of the observed samples in ways that we can exploit.

Our analysis relies on a comparison of three different covariance matrices: (a) the sample covariance matrix $\sCov$, defined in \eqref{eq:sample_cov}; (b) the ensemble covariance matrix $\bm S(T)$, defined in \eqref{eq:ensemble}; and (c) the covariance matrix $\Cov$, defined by
\begin{align}
    \Cov &:= \mathbb{E}_{\bm x_0}[\bm x_T \bm x_T^\top]- \mathbb{E}_{\bm x_0}[\bm x_T] \mathbb{E}_{\bm x_0}[\bm x_T^\top] = [\nA^T][\nA^T]^\top = \nA^{2T},
\end{align}
where the first equality uses the zero-mean property of $\bm x_T$, and the last equality uses the symmetry of $\nA$.

As a first step, we establish that the sample covariance $\sCov$ serves as a good estimator of  $\Cov$.  This is done by leveraging the fact
that the effective rank $r_e(\Cov) = \text{trace}(\Cov) / \|\Cov\|$ of  $\Cov$ is low, and then showing that $\sCov$ indeed concentrates around $\Cov$, even when the number of samples is relatively small (\Cref{T:boundSampleCovarianceToCovariance}).

Second, we observe that the eigenvectors of $\Cov$ are the same as the eigenvectors of $\nA$, although they now correspond  to the modified eigenvalues $\lambda_i(\Cov) = \lambda_i^{2T}(\nA)$.
This transformation does not change the ordering of the eigenvalues in terms of their absolute values and thus the dominant eigenspaces of $\Cov$ and $\nA$ are the same. In particular this implies that the dominant eigenspace of $\Cov$ allows us to approximate the dominant eigenspace of $\nA$.

Finally, we show that $\nA$ concentrates around $\nEA$ (cf.~\Cref{T:concentration_A}), so that $\Cov$ also concentrates around $\bm S(T)$. In particular, the dominant eigenspace of $\Cov$ allows us to approximate the dominant eigenspace of $\bm S(T)$ and of $\nEA$, which in turn allows us to estimate the partition structure.

From the above discussion we see that~\Cref{alg} is subject to two sources of randomness: (a) the normalized adjacency matrix $\nA$ is an approximation of its expected version $\nEA$; and (b)  
the sample covariance matrix $\sCov$ is an approximation of the true covariance matrix $\Cov$. 
\Cref{T:main_result,T:main_result_gen_f} allow us to bound these two effects, leading to the performance guarantees in Theorem~\ref{T:result_misclassification}.

\begin{remark}[Computational considerations]
In \Cref{alg}, we first form the sample covariance and then analyze its spectral properties via an eigenvalue decomposition.
Equivalently, we can employ an SVD of the normalized and centered sample matrix $\bm X = \frac{1}{\sqrt{s}}[\bm{\widetilde{x}}_T^{(1)},\ldots,\bm{\widetilde{x}}_T^{(s)}]$, where $\bm{\widetilde{x}}_T^{(i)} = \bm x_T^{(i)} - \bm{\bar{x}}_T$. 
Using the SVD $\bm X = \bm U \bm \Theta  \bm Q^\top$ 
it is easy to see that $\sCov = \bm X \bm X^\top = \bm U \bm \Theta^2 \bm U^\top $ and hence all the information required by the algorithm is contained in the SVD of $\bm X$.
While mathematically equivalent, working directly with the SVD of $\bm X$ avoids forming $\sCov$ explicitly, and thus reduces the memory requirements when $s \ll n$.
\end{remark}

\section{Parameter estimation}\label{sec:parameter_rec}
In this section we turn to the parameter estimation problem.
We first provide an overview of the general algorithmic idea with the help of an example.
We then discuss certain obstacles to its solution, including some insurmountable non-identifiability issues.
In \Cref{sec:parameter_algo} we then provide a generic ``meta''-algorithm for the solution of the parameter identification problem. Finally, in \Cref{sec:th_param_est},
we provide some theoretical guarantees for its performance, for a planted partition model, which is a special case of the SBM.

\subsection{Algorithm ideas}
To estimate the \emph{parameters} of the underlying SBM we can use a similar strategy as for the recovery of the \emph{partitions}. 
The ensemble covariance matrix $\bm S(T)= f(\EA)^{2T}$ is a transformation of the normalized adjacency matrix.
Hence, it contains relevant information not only about the partitions, but also about the model parameters.
Let us illustrate this fact with our example dynamics $f(\bm A) = \nA$ on a graph drawn from of a planted partition model, a simple form of an SBM.

\begin{example}\label{ex:example1}
	Consider a special case of the SBM, known as the planted partition model.
    Conditional on the group labels, the planted partition model with $n$ nodes is described by probabilities $a/n$ and $b/n$ (with $a \neq b$) of linking to a node in the same group or a different group, respectively.\footnote{In~\Cref{T:parameter_estimation_result} below, we will assume that $a, b$ and $n$ are such that $\mathcal{M}(\bm{\mathcal{A}}) \in \mathcal{M}_n(\epsilon)$. In the context of this example, this means $a$ and $b$ are assumed to be of order at most $\ln(n)$.}
	The affinity matrix for such a model, with $k$ blocks, is of  the form $\bm \Omega = \frac{(a-b)}{n}\bm I_k + \frac{b}{n}\bm 1^{}_k\bm 1_k^\top$, where $\bm I_k$ and $\bm 1_k$ denote the $k$-dimensional identity matrix and the vector of all ones, respectively.
    Let us consider such a planted partition matrix with equally-sized groups.
	This leads to an expected matrix $\nEA$ with entries
	\begin{equation*}
	[\nEA]_{ij} = 
	\begin{cases}
	\frac{a}{n}\cdot\frac{k}{a+(k-1)b}, & \text{if } g_i = g_j,\\
	\frac{b}{n}\cdot\frac{k}{a+(k-1)b},& \text{if } g_i \neq g_j,
	\end{cases}
	\end{equation*}
	where, we recall, $g_i$ represents the latent group label of node $i$.
	
	Based on these calculations we can find expressions for ${\bm S}(T)=\nEA^{2T}$:
	\begin{equation}\label{E:example_entries}
	[\nEA^{2T}]_{ij} = 
	\begin{cases}
	\frac{k-1}{n}\left(\frac{a-b}{a+(k-1)b}\right)^{2T} + \frac{1}{n} & \text{if } g_i = g_j,\\
	-\frac{1}{n}\left(\frac{a-b}{a+(k-1)b}\right)^{2T} + \frac{1}{n} & \text{if } g_i \neq g_j.
	\end{cases}
	\end{equation}
\end{example}

This example suggests the following generic strategy for the parameter recovery problem.
First, we can derive the functional dependencies of the entries of $\bm S(T)$ on the latent parameters.
Let us denote these relations by $F(\bm \Omega, \bm G) = \bm S(T)$. 
In the above example, $F(\bm \Omega, \bm G)$ is given by~\Cref{E:example_entries}.
Note that it is possible to derive such relationships $F(\bm \Omega, \bm G) = \bm S(T)$ for any SBM with generic parameters $\bm \Omega$ and $\bm G$, even though the functional relationships will be more complicated. 
Second, in principle, we can then use the sample covariance matrix $\sCov$ as an approximation of $\bm S(T)$ and estimate the parameters by leveraging the functional relationships $F(\bm \Omega, \bm G) \approx \sCov$.
However, there are some difficulties to overcome to render this general idea operational.

First, the derived equations might be insufficient to identify the model parameters.
Indeed, notice that the equations in~\eqref{E:example_entries} do not enable us to identify both model parameters $a$ and $b$ of the planted partition model.
This is an artifact of the symmetry properties of the normalized adjacency matrix $\bm L$.
In particular, this degeneracy arises from an invariance in the dynamical model~\eqref{eq:observation_model}. 
Indeed, if in \Cref{ex:example1} we replace $(a,b)$ by $(\gamma a,\gamma b)$, where $\gamma$ is a positive constant, the expected matrix $\nEA$ remains the same. 
Thus the model parameters are not recoverable, even if we had access to the exact matrix $\nEA$.
In the context of our specific model~\eqref{eq:observation_model}, we thus have to make some additional assumptions to obtain a viable parameter estimation approach.\footnote{For a planted partition model with two equally-sized blocks there is additional symmetry in the problem in that we can exchange $(a,b)$ with $(b,a)$ and obtain exactly the same set of equations.
In this case an additional assumption, e.g., of the form $a > b$ is thus necessary for identifiability.}
We remark that for other dynamical models such additional assumptions may not be necessary.

A second issue is that the obtained equations will be typically inconsistent if we just plug in the sample covariance matrix. Thus,
instead of simply inverting the functional relations $F(\bm \Omega, \bm G) \approx \sCov$, we will have to find approximate solutions.
There are many ways of obtaining an approximate solution in such cases, e.g., by formulating an optimization problem $\min_{\bm \Omega} \|F(\bm \Omega, \bm G) - \sCov\|$.
Which approach would have better statistical properties in the face of the noise present in $\sCov$ is an interesting question.
We do not go into such details, but provide instead a generic ``meta''-algorithm in the next section.

A third possible issue with the above approach is that it requires the knowledge of the group labels.
In the context of Example~\ref{ex:example1}, based on~\eqref{E:example_entries} we can estimate the parameters $a$ and $b$ directly from the entries of $\bm S(T)$ -- or rather its estimated version $\sCov$ -- but only if the group labels $g_i$ of the nodes are known.
This implies that we would have to first recover the latent partition before we can estimate the parameters based on these relations.
Although this is a reasonable procedure in practice, the recovery performance will depend strongly on how good the partition recovery is in the first place.

However, in certain cases we can avoid the latter problem as follows.
In particular, we can express the eigenvalues of $\bm S(T)$ as a function of the model parameters, and obtain a set of equations $\lambda(\bm{S}(T))= F_\lambda(\bm\Omega)$ for the parameters $\bm\Omega$.
Notice that these equations are independent of the group labels of the nodes.
The reason is that a permutation of the nodes defines a similarity transformation of $\nEA$, and hence the eigenvalues of $\nEA$ -- and accordingly those of $\bm S(T)$ -- are not affected by the ordering of the rows and columns of the matrix.

To make these ideas concrete let us reconsider our earlier example.
\begin{example}[continued]\label{ex:example1_cont}
	In terms of spectral analysis, the nonzero eigenvalues of $\bm S(T)$ are:
	\begin{equation}\label{E:example_spectral}
        \lambda_i(\bm S(T)) = 
	\begin{cases}
	 1, & \text{if } i=1,\\
     \left(\frac{a-b}{a+(k-1)b}\right)^{2T}, & \text{if } i=2,\ldots,k.
	\end{cases}
	\end{equation}
\end{example}

As expected, these equations are independent of the group labels.
However, as outlined in our discussion above, as the equations are redundant, on their own they are insufficient to identify all the parameters, but some additional side-information is needed.
Nevertheless, the fact that these equations are independent of the group labels makes this ``eigenvalue-based'' parameter inference strategy more amenable to theoretical analysis, as we will see in~\Cref{sec:th_param_est}.

\subsection{A meta-algorithm for parameter estimation}\label{sec:parameter_algo}
Despite the difficulties identified in the previous section, we can still obtain a viable parameter estimation approach.
To avoid issues of non-identifiability, we will have to assume some additional, partial prior  knowledge on the model parameters.
More concretely, suppose that we know of some additional constraints on $\bm \Omega$ (or alternatively $\bm{\mathcal{A}}$), of the form $h(\bm\Omega)=0$, for some known function $h$.
Such constraints could be, e.g., on the expected density $\rho$ of the network, in which case $h(\bm \Omega)$ would be of the form 
\begin{equation*}
    h(\bm \Omega) = \rho -  \frac{\bm 1^\top\EA \bm 1}{n^2}  = 0.
\end{equation*}

As discussed earlier, based on our dynamical model we can obtain relations of the form ${\bm S(T) = F(\bm\Omega, \bm G)}$ for the entries of the ensemble covariance matrix and $\lambda(\bm S(T)) = F_\lambda(\bm \Omega)$ for the eigenvalues.
Taking these equations and the constraints $h(\bm \Omega)$ together, we obtain a system of (generally redundant) equations of the form
    \begin{subequations}\label{eq:parameter_estimation}
\begin{align}
    \bm S(T) &= F(\bm\Omega,\bm G)\label{eq:part_based_estimation}\\ 
    \lambda(\bm S(T)) &= F_\lambda(\bm \Omega)\label{eq:ev_based_estimation}\\
    h(\bm\Omega) &= 0.
\end{align}
\end{subequations}

Our approach is to ``solve'' this system of equations, while replacing $\bm S(T)$ by its estimate $\sCov$, as outlined in \Cref{alg2}.
Observe that the sampling time $T$ enters explicitly in the functional relations that we exploit, and therefore, unlike the partition recovery problem (cf. Algorithm~\ref{alg}), $T$ must be known.

As discussed in the previous section, the word ``solve'' in the description of \Cref{alg2}  is to be understood as some not fully specified approach that yields an approximate solution of the system~\eqref{eq:parameter_estimation}.

\begin{algorithm}[thb!]
    \caption{\vspace{1pt} Model parameter estimation\label{alg2}}
    \begin{algorithmic}[1]
    \item[]{\textbf{Input:} Samples $\{\bm x_T^{(i)}\}_{i=1}^s$, Number of groups $k$, Time $T$}
    \item[]{\textbf{Output:} Model Parameters $\bm\Omega$}
    \STATE{Compute the sample covariance $\sCov := \frac{1}{s}\sum_{i=1}^s [\bm x_T^{(i)} - \bm{\bar x}_T][\bm x_T^{(i)}- \bm{\bar x}_T]^\top$}\\
    \STATE{``Solve'' the set of equations
    \begin{align*}
        \sCov &= F(\bm\Omega,\bm G)\\ 
        \lambda(\sCov) &= F_\lambda(\bm \Omega)\\
         h(\bm\Omega) &= 0.
    \end{align*}
    to estimate model parameters $\bm \Omega$ }\\
\end{algorithmic}
\end{algorithm}

\subsection{Theoretical guarantees for a planted partition model}\label{sec:th_param_est}
In this section, we provide some theoretical results for a parameter estimation strategy in which we will ignore~\Cref{eq:part_based_estimation} and concentrate on recovering the parameters solely based on the eigenvalues and the constraints.
To this effect, we will leverage our concentration result in~\Cref{T:main_result} to provide a bound between the estimated eigenvalues $\lambda_i(\sCov)$ and their ensemble counterparts $\lambda_i(\bm S(T))$.
These bounds then translate immediately into ($T$-dependent) error-bounds on the parameter recovery.

For simplicity, we illustrate our approach by considering our example dynamics $f(\EA) = \bm L$ on a graph drawn from a two-group planted partition model (see~\Cref{ex:example1,ex:example1_cont}) subject to a density constraint on the network and the assumption that $a>b$.
Similar results could be derived for other dynamics as well, using analogous arguments, but they would depend  on the (possibly more complicated) parametrization of the model.

Recall from~\Cref{ex:example1} that  the second eigenvalue of $\nEA^{2T}$ satisfies 
$$(\lambda_2(\nEA^{2T}))^{1/2T}= \frac{a-b}{a+(k-1)b}.$$
Given also prior knowledge of the density $\rho = \frac{a+(k-1)b}{nk}$, we can recover the parameter $a$, for the case $k=2$ as follows:
\begin{equation}\label{E:estimate_a}
a = (\lambda_2(\nEA^{2T}))^{1/2T} \rho n + \rho n.
\end{equation}
Then, to estimate $a$ from samples, we can  replace $\nEA^{2T}$ in \eqref{E:estimate_a} by the sample covariance matrix $\sCov$.
For this estimation strategy we can derive the following error bound.

\begin{theorem}\label{T:parameter_estimation_result}
	Consider an SBM $\mathcal M(\bm{\mathcal{A}}) \in \mathcal M_n(\epsilon)$ with two equally-sized blocks, and a parameter matrix of the form $\bm \Omega = \frac{(a-b)}{n}\bm I_2 + \frac{b}{n}\bm 1^{}_2\bm 1_2^\top$, with $a > b$.
    Assume that $\rho = (a+b)/(2n)$  is given.
	Consider the estimator of $a$ given by
    \begin{equation}\label{eq:estimator_lambda}
	\hat{a} = (\lambda_2(\sCov))^{1/2T} \rho n + \rho n.
	\end{equation}
	Then, using the notation of \Cref{T:main_result}, the estimation error $\eta := |\hat{a} - a| / a$ satisfies the upper bound
	\begin{equation}\label{E:theo_parameter_estimation}
	\eta \le \frac{(2TM +B(T) )^{1/2T}}{\lambda_2(\nEA) + 1},
	\end{equation}
	with probability at least $1-5s^{-1}-2\epsilon$.
\end{theorem}
The proof of \Cref{T:parameter_estimation_result} is given in~\Cref{sec:implications_inference_parameters}.

Note that  the numerator in \eqref{E:theo_parameter_estimation} indicates that the sampling time $T$ plays a fundamental role in the estimation accuracy of $\hat{a}$. 
Moreover, this dependency is similar to the one encountered in Theorem~\ref{T:result_misclassification} for partition recovery. In particular, there is a term that increases linearly with $T$, i.e., $2TM$, and another term that may decrease with $T$, i.e., $B(T)$ (cf.~the discussion after Theorem~\ref{T:main_result}).
The denominator in \eqref{E:theo_parameter_estimation} can also be explained intuitively.
A larger $\lambda_2(\nEA)$ is associated with a marked difference between $a$ and $b$ [cf.~\eqref{E:example_spectral}], which translates into better delineated blocks within our SBM. 
In this scenario, estimating $a$ should be easier, and this is captured in~\eqref{E:theo_parameter_estimation}.

\begin{remark}
    The above outlined eigenvalue-based parameter inference strategy is independent of the partition recovery problem and thus enables a clean theoretical analysis of the estimation error.
    Had we considered instead the general approach based on \Cref{eq:part_based_estimation}, we would be solving simultaneously the partition recovery and the parameter estimation problems, and the estimation error would be affected by partition recovery errors. While a theoretical analysis of such an interplay 
may be feasible in principle, we do not pursue it further here, but instead illustrate it with the help of numerical experiments in~\Cref{sec:numerical_exp}.
\end{remark}

\section{Theoretical analysis}\label{sec:theoretical_analysis}

In the following, we provide the theoretical analysis of the spectral inference algorithms (\Cref{alg,alg2}) discussed above.
We first prove the concentration result in~\Cref{T:main_result}, which underlies the success of our spectral algorithms.
The proof of this result follows from a series of lemmas that we establish in~\Cref{sec:bound_second_Term,sec:bound_first_Term}.
Afterwards, we show in~\Cref{sec:implications_inference} how~\Cref{T:main_result} implies the partition recovery performance guarantees stated in~\Cref{T:result_misclassification}. We finally  prove \Cref{T:parameter_estimation_result} in \Cref{sec:implications_inference_parameters}.

\subsection{Proof of~\Cref{T:main_result}: Concentration of the sample covariance matrix}\label{sec:concentration_of_covariance}
In the next two subsections we prove~\Cref{T:main_result} through a series of lemmas. 
The proof uses the triangle inequality 
\begin{equation}\label{E:main_Triangle_inequality}
\|\sCov - \bm S(T) \| \le \|\sCov - \Cov\| + \|\Cov - \bm S(T)\|,
\end{equation}
and separate bounds for the two terms on the right-hand side.

\subsubsection{Bounding the first term in~\eqref{E:main_Triangle_inequality}}\label{sec:bound_first_Term} 
We now establish a bound for the distance ${\|\sCov - \Cov\|}$ between the true and the sample covariances.

\begin{lemma}\label{T:boundSampleCovarianceToCovariance}
    Under the conditions in~\Cref{T:main_result}, 
    we have
    \begin{equation}\label{eq:bt}
        \left \| \sCov - \Cov \right \| \le B(T),
    \end{equation}
    with probability at least $1-5s^{-1}-\epsilon$, 
    where $B(T)$ is as defined in~\Cref{T:main_result}. 
\end{lemma}

For the proof of~\Cref{T:boundSampleCovarianceToCovariance} we make use of the following result established in~\cite{bunea2015sample}.
\begin{lemma}{\cite[Theorem 2.1]{bunea2015sample}}\label{T:boundSampleCovarianceToCovarianceBunea}
    Let $\bm x \in \mathbb{R}^n$ be a zero-mean sub-Gaussian vector satisfying~\Cref{assumption2}, with covariance matrix $\bm\Sigma$. 
    Let $\bm x^{(1)},\ldots,\bm x^{(i)}$ be i.i.d samples drawn from $\bm x$.
    Consider the sample covariance matrix ${\widehat{\bm\Sigma}:= \frac{1}{s}\sum_{i=1}^s [\bm x^{(i)} - \bm{\bar x}][\bm x^{(i)}- \bm{\bar x}]^\top}$, where ${\bm{\bar x} := \frac{1}{s}\sum_{i=1}^s\bm x^{(i)}}$ is the sample mean.
    Then, with probability at least $1-5s^{-1}$ we have 
    \begin{equation}
        \|\widehat{\bm\Sigma} - \bm\Sigma\|_F \le c \cdot \|\bm\Sigma\| \cdot r_e(\bm\Sigma) \cdot \sqrt{\dfrac{\ln s}{s}},
    \end{equation}
    where $r_e(\bm\Sigma) = {\sum_i \lambda_i(\bm\Sigma)}{/}{\|\Sigma\|}$ is the effective rank of $\bm\Sigma$, $c$ is a constant that depends on $c_0$ in Assumption~\ref{assumption2}, and $\|\cdot\|_F$ denotes the Frobenius norm.
\end{lemma}
Notice that we always have $\|\widehat{\bm\Sigma} - \bm\Sigma\| \le \|\widehat{\bm\Sigma} - \bm\Sigma\|_F$ and, for the scenario considered here, $\|\Cov \| = 1$, as by construction the dominant eigenvalue of $\bm L$ equals $1$. Thus, \Cref{T:boundSampleCovarianceToCovarianceBunea}  provides a useful tool to prove~\Cref{T:boundSampleCovarianceToCovariance}, provided that we can find an upper bound on the effective rank $r_e(\Cov)$.

\begin{lemma}[Bound on the effective rank]\label{T:boundEffectiveRank}
    Under the conditions stated in~\Cref{T:main_result},  the effective rank satisfies the bound
    \begin{equation}
       r_e(\Cov) \le 1 + \sum_{i=2}^k\min\{1, |\elambda_i| + M\}^{2T} + (n-k)M^{2T},
    \end{equation}
    with probability at least $1-\epsilon$.
\end{lemma}

\begin{proof}
    Recall that we denote the eigenvalues of $\nA$ by $\lambda_i$ and the eigenvalues of $\nEA$ by $\elambda_i$.
   Recall our convention that the eigenvalues have been sorted in descending order, according to their magnitude, i.e., $|\lambda_1| \ge |\lambda_2| \ge \cdots \ge |\lambda_n|$ and $|\elambda_1| \ge |\elambda_2| \ge \cdots \ge |\elambda_n|$, respectively.
    From the definition of the effective rank and the fact that $\|\Cov\| =1$ it follows that
    \begin{align*}
        r_e(\Cov) = \sum_{i=1}^n \lambda_i(\Cov) = \sum_{i=1}^n \lambda_i^{2T} = \sum_{i=1}^n |\lambda_i|^{2T}.
    \end{align*}
    We can decompose this sum into three parts as follows
    \begin{align*}
        r_e(\Cov) = 1 + \sum_{i=2}^k |\lambda_i|^{2T} + \sum_{j=k+1}^n|\lambda_j|^{2T}.
    \end{align*}
    In order to relate the (random) eigenvalues $\lambda_i$ to quantities that can be expressed in terms of the model parameters, we make use of the triangle inequality ${|\lambda_i| \le |\elambda_i| + |\lambda_i -\elambda_i|}$.
    Weyl's Theorem (see e.g.,~\cite{horn1990matrix}) states that $|\lambda_i- \elambda_i | \le \|\nA -\nEA\|$ for all $i$. 
    Furthermore, \Cref{T:concentration_A} implies that $\|\nA-\nEA\| \le M$, with probability at least ${1-\epsilon}$.
Finally, we also know that $|\lambda_i| \le 1$ for all $i$, since $\|\bm L\|=1$ (more generally, we assumed $\|f(\bm A)\| \le 1$). 
    We can thus deduce that $|\lambda_i| \le \min\{1, |\elambda_i| + M\}$ for all $i$, with probability at least $1-\epsilon$.
Using in addition the fact that  $\mu_i = 0$ for $i>k$ and $M < 1$ if $\mathcal M(\bm{\mathcal{A}}) \in \mathcal M_n(\epsilon)$, we conclude that
    \begin{align*}
        r_e(\Cov) &\le 1 + \sum_{i=1}^k \min\{1, |\elambda_i| + M\}^{2T}+\!\! \sum_{j=k+1}^n\!\!\min\{1, |\elambda_i| + M\}^{2T} \\
                  & \le  1 +\! \sum_{i=2}^k \min\{1, |\elambda_i| + M\}^{2T} + (n\!-\!k) M^{2T},
    \end{align*}
with probability at least $1-\epsilon$.
\end{proof}

\begin{proof}[Proof of \Cref{T:boundSampleCovarianceToCovariance}]
    By combining Lemmas~\ref{T:boundSampleCovarianceToCovarianceBunea} and \ref{T:boundEffectiveRank}, \Cref{T:boundSampleCovarianceToCovariance} follows.
\end{proof}

\subsubsection{Bounding the second term in~\eqref{E:main_Triangle_inequality}}\label{sec:bound_second_Term}
We derive a bound for the difference $\|\Cov - \bm S(T)\|$, between the true and the ensemble covariances.
We start with the following lemma.
\begin{lemma}\label{T:ConcentrationPowersOfMatrix}
    For any integer $t>0$, we have 
    \begin{equation}\label{E:lemma_iterative}
    \|\nA^t - \nEA^t\| \le t \|\nA - \nEA\|.
    \end{equation}
\end{lemma}
\begin{proof}
    We will show that for any integer $t>0$,
    \begin{equation}\label{E:lemma_iterative_proof_010}
        \|\nA^t - \nEA^t\| \le \|\nA - \nEA\| + \|\nA^{t-1} - \nEA^{t-1}\|.
    \end{equation}
    \Cref{E:lemma_iterative} then  follows directly from repeated application of \eqref{E:lemma_iterative_proof_010}. 
    To show \eqref{E:lemma_iterative_proof_010}, notice that
    \begin{align*}
        \left\|\nA^t - \nEA^t\right\| &= \left\| \left(\nA - \nEA\right)\nA^{t-1} + \nEA\left(\nA^{t-1}-\nEA^{t-1}\right) \right\| \\
                                      & \le \left\| \nA - \nEA\right\|\|\nA^{t-1}\| + \|\nEA\|\left\|\nA^{t-1}-\nEA^{t-1} \right\|
                                      = \left\| \nA - \nEA\right\| + \left\|\nA^{t-1}-\nEA^{t-1} \right\|,
    \end{align*}
    where the last equality follows from the fact that $\|\nEA\| = \|\nA\| = 1$.
\end{proof}

\begin{proof}[Concluding the proof of \Cref{T:main_result}]
Recall that $\bm S(T)=\nEA^{2T}$ and $\Cov=\nA^{2T}$, and that 
$\left\| \nA - \nEA\right\|\leq M$, with probability at least $1-\epsilon$ (\Cref{T:concentration_A}).
 \Cref{T:ConcentrationPowersOfMatrix} then implies that
 \begin{equation}\label{eq:two}
       \|\Cov - \bm S(T)\| \leq 2T \|\nA - \nEA\| \le 2TM,
    \end{equation}
 with probability at least $1-\epsilon$.
We bound the two terms on the right-hand side of \eqref{E:main_Triangle_inequality}
using \eqref{eq:bt} and \eqref{eq:two}, respectively, and hence establish the validity of  \Cref{T:main_result}.
\end{proof}

\subsection{Proof of \Cref{T:main_result_gen_f}}\label{sec:proof_main_gen_f}
\begin{proof}
    The proof is analogous to the proof of \Cref{T:main_result}.
    Instead of using a bound for $\|\nA -\nEA\|$, we develop a bound for $\|f(\bm A) - f(\EA)\|$, and obtain that, with probability at least $1-5s^{-1} -\epsilon$,
    \begin{align}
        \|f(\bm A) - f(\EA)\| = \left \| \sum_{\ell=1}^R \alpha_\ell (\nA^\ell - \nEA^\ell) \right\| \leq \sum_{\ell=1}^R\alpha_\ell \left \| \nA^\ell - \nEA^\ell \right\|
        \leq \sum_{\ell=1}^R\ell\alpha_\ell  \left \| \nA - \nEA \right\| = M_f,
    \end{align}
    where the last inequality follows from~\Cref{T:ConcentrationPowersOfMatrix}.
    The rest of the argument remains the same, but with $M_f$ replacing $M$.
\end{proof}

\subsection{Implications for partition recovery}\label{sec:implications_inference}
In the following, we explain how the above concentration results translate into guarantees for the partition recovery problem. 
\Cref{sec:p_inference_prelim} provides intermediate results on the spectral geometry of the problem and discusses the precise definition of misclassified nodes.
The proof of~\Cref{T:result_misclassification} is then given in \Cref{sec:proof_p_recovery}.

\subsubsection{Partition recovery problem}\label{sec:p_inference_prelim}
In order to characterize recovery performance, we first clarify what we mean by a node being clustered correctly. 
To do this we follow the literature on SBMs~\cite{rohe2011spectral,qin2013regularized}.
Recall that within our partition recovery algorithm, we perform a $k$-means clustering procedure on the rows of the matrix $\widetilde{\bm W}$ of (row-normalized) dominant eigenvectors associated with the estimated covariance $\sCov$.
This procedure assigns each row $i$ to a centroid vector $\bm k_i$ corresponding to a cluster.

Let us denote the `true' centroid vector that a node $i$ is assigned to according to the ensemble covariance $\bm S(T)$ by $\bm \kappa_i$, where $i=1,\ldots,n$.
From the definition of $\bm S(T)$ it follows that each $\bm \kappa_i^\top$ corresponds to a row of the  matrix $\widetilde{\bm V}$ of the (row-normalized) top $k$ eigenvectors of $\nEA$.
Moreover, we know from \Cref{T:eigenvectors_and_partitions} that there are exactly $k$ different population centroids and that the association of rows to centroids corresponds to the true group structure of the underlying SBM.
However, if some of the eigenvalues of $\bm S(T)$ are repeated, the eigenvector matrix $\widetilde{\bm V}$ will not be uniquely defined, even though the subspace spanned by these vectors will be unique.
As we are interested in the distance between the vectors $\widetilde{\bm V}$ and their approximation via $\widetilde{\bm W}$, our definition of misclassification has to include an orthogonal matrix $\bm{\mathcal{Q}} \in \mathbb{R}^{k\times k}$ that minimizes the orthogonal Procrustes distance ${\|\widetilde{\bm W} - \widetilde{\bm V}\bm{\mathcal{Q}}\|_F}$ (see also~\cite{rohe2011spectral,qin2013regularized}).
In particular, this minimality property of $\bm{\mathcal{Q}}$ will be essential for~\Cref{T:Davis-Kahan}.

\begin{definition}[Set of misclassified nodes \cite{qin2013regularized}]\label{D:mis_nodes}
    The set of misclassified nodes $\mathbb{M}$ is defined as
    \begin{equation}\label{eq:misclassified_nodes}
        \mathbb{M} := \lbrace  i : \|\bm{k}_i^\top - \bm{\kappa}_i^\top\bm{\mathcal{Q}}\| > \|\bm{k}_i^\top - \bm{\kappa}_j^\top\bm{\mathcal{Q}}\| \text{ for some } j\neq i \rbrace.
    \end{equation}
\end{definition}
In words, this means that we call a node $i$ misclassified if the centroid $\bm k_i$ to which it has been assigned is closer to some other population centroid $\bm\kappa_j$ than it is to its true population centroid $\bm \kappa_i$.
Note that this definition of misclassified nodes adopted from the literature on SBMs~\cite{rohe2011spectral,qin2013regularized} implicitly assumes that i) we can solve the $k$-means clustering problem exactly and ii) the empirical centroids based on $\widetilde{\bm W}$ will be centered around the (appropriately rotated) population centroids.
We refer to the literature for additional discussion on the consistency of spectral clustering under the SBM~\cite{lei2015consistency}, and limit theorems for the eigenvectors of random graphs (see, e.g.,~\cite{athreya2017statistical} for a survey on random dot product graphs).

By accounting for the special geometry of the eigenvectors $\widetilde{\bm V}$, we can establish the following necessary condition for a node to be misclassified.

\begin{lemma}\cite{qin2013regularized}\label{T:M2misclassified_nodes}
    If a node $i$ is misclassified, then ${\|\bm{k}_i^\top - \bm{\kappa}_i^\top\bm{\mathcal{Q}}\| \ge 1/\sqrt{2}}$.
\end{lemma}
\begin{proof}
    First, we know from~\Cref{T:eigenvectors_and_partitions} that $\widetilde{\bm V} = \bm G \bm{\mathcal U}$ for some orthogonal matrix $\bm{\mathcal U}$.
    Hence, all population centroids $\bm \kappa_i$ are of unit length and orthogonal to each other.
    In particular, for all $j$ and $i$ belonging to different blocks/groups, we have
    \begin{equation}\label{eq:sqrt2distance}
        \|\bm{\kappa}_i^\top-\bm{\kappa}_j^\top\| = \|(\bm{g}_i^\top - \bm{g}_j^\top)\bm{\mathcal U}\| = \sqrt{2},
    \end{equation}
    where $\bm{g}_i$ is the true population cluster indicator vector corresponding to node $i$ (see~\Cref{T:eigenvectors_and_partitions}).

    Second, by the triangle inequality and the orthogonality of $\bm{\mathcal{Q}}$, we know that for all $j$ and $i$ in different blocks, we have
    \begin{equation}\label{eq:triangle_centroids}
        \|\bm{\kappa}_i^\top-\bm{\kappa}_j^\top\| \le \|\bm{\kappa}_i^\top\bm{\mathcal{Q}}-\bm{k}_i^\top\| + \|\bm{k}_i^\top -\bm{\kappa}_j^\top\bm{\mathcal{Q}}\|.
    \end{equation}
    
    Plugging in the result from~\eqref{eq:sqrt2distance} into~\eqref{eq:triangle_centroids}, we obtain that $\|\bm{k}_i^\top - \bm\kappa_j^\top\bm{\mathcal Q}\| \ge \sqrt{2} - \|\bm{k}_i^\top - \bm\kappa_i^\top\bm{\mathcal Q}\|$. 
    From Definition~\ref{D:mis_nodes}, for a misclassified node $i$ there exists some $j$ such that
    \begin{equation*}
        \|\bm{k}_i^\top - \bm{\kappa}_i^\top\bm{\mathcal{Q}}\| > \|\bm{k}_i^\top - \bm{\kappa}_j^\top\bm{\mathcal{Q}}\| \ge \sqrt{2} - \|\bm{k}_i^\top - \bm\kappa_i\bm{\mathcal Q}\|,
    \end{equation*}
    where the last inequality is true by our argument above.
    The result follows after some simple algebraic manipulations.
\end{proof}

To arrive at our result for the misclassification rate in~\Cref{T:result_misclassification}, we further follow closely the literature on SBMs~\cite{rohe2011spectral,qin2013regularized,Abbe2018}, and establish the following two lemmas (Lemmas~\ref{T:KmatrixVsV} and \ref{T:Davis-Kahan}).
The first lemma provides us with a relationship between the difference of the cluster centroids obtained from the $k$-means clustering and the true population cluster on the one side, and the difference between the eigenvectors of $\sCov$ and $\bm S(T)$ on the other side.
\begin{lemma}\label{T:KmatrixVsV}
    Denote the matrix of cluster centroids obtained by performing k-means on the rows of $\widetilde{\bm W}$ as ${\bm K := [\bm k_1, \ldots, \bm k_n]^\top \in \mathbb{R}^{n \times k}}$.
    Then,
    \begin{equation}
        \|\bm K - \widetilde{\bm V} \bm{\mathcal Q}\|_F \le 2 \|\widetilde{\bm W}  - \widetilde{\bm V}\bm{\mathcal Q}\|_F.
    \end{equation}
\end{lemma}
\begin{proof}
    By the triangle inequality,
    \begin{align*}
        \|\bm K - \widetilde{\bm V} \bm{\mathcal Q}\|_F \le \|\bm K - \widetilde{\bm W}\|_F + \|\widetilde{\bm W} -\widetilde{\bm V}\bm{\mathcal Q}\|_F.
    \end{align*}
    From the definition of $k$-means, however, we know that $\|\widetilde{\bm W} - \bm K\|_F \le \|\widetilde{\bm W} -\widetilde{\bm V}\bm{\mathcal Q}\|_F$ as $\bm K$ is chosen in order to minimize the distances of the centroids to the individual rows.
    Hence, the claim follows.
\end{proof}

The second lemma builds on the Davis-Kahan theorem~\cite{davis1970rotation,yu2014useful}, and bounds the difference between the row-normalized eigenvectors of $\sCov$ and $\bm S(T)$.

\begin{lemma}[Davis-Kahan]\label{T:Davis-Kahan}
    Following the definitions in~\Cref{T:result_misclassification}, we have
    \begin{equation}
        \|\widetilde{\bm W}  - \widetilde{\bm V}\bm{\mathcal Q}\|_F \le \frac{\sqrt{8k}\|\sCov - \bm S(T)\|}{\tau \xi_k}.
    \end{equation}
    where we recall (see \Cref{T:result_misclassification}) $\xi_k$ is the $k$-th top eigenvalue of the ensemble covariance matrix, and $\tau$ is be the minimum 2-norm of any row of $\bm{W}$ and $\bm{V}$.
\end{lemma}
\begin{proof}
    The fact that
    \begin{equation*}
        \|\bm W  - \bm V\bm{\mathcal Q}\|_F \le \frac{\sqrt{8k}\|\sCov - \bm S(T)\|}{\xi_k},
    \end{equation*}
    holds for the unnormalized versions of $\bm W$ and $\bm V$ is a restatement of the Davis-Kahan Theorem as can be found in, e.g.,~\cite[Theorem 2]{yu2014useful}.

    By definition, we have $\widetilde{\bm V} = \diag(\bm V \bm V^\top)^{-1/2}\bm V$ and $\widetilde{\bm W} = \diag(\bm W \bm W^\top)^{-1/2}\bm W$.
    Let us define $\tau_i := \min\{[\bm V \bm V^\top]_{ii}^{1/2},[\bm W \bm W^\top]_{ii}^{1/2}\}$, which is the smaller of the 2-norms of the $i$-th rows of $\bm V$ and $\bm W$, respectively.
    Based on this definition of $\tau_i$ we see that $\tau = \min_i \tau_i$.
    Now, as row-normalization corresponds to diagonal scaling we obtain
    \begin{equation*}
        \|\widetilde{\bm W}  - \widetilde{\bm V}\bm{\mathcal Q}\|_F \le \frac{1}{\tau}\|\bm W  - \bm V\bm{\mathcal Q}\|_F \le \frac{\sqrt{8k}\|\sCov - \bm S(T)\|}{\tau\xi_k}.
    \end{equation*}
\end{proof}

\subsubsection{Proof of \Cref{T:result_misclassification}}\label{sec:proof_p_recovery}
Based on the results established thus far, we can provide a proof of the bound of the misclassification rate $q :=|\mathbb{M}|/n$ as stated in~\Cref{T:result_misclassification}.
Our proof parallels results for the misclassification rate if the graph is observed~\cite{qin2013regularized}.
\begin{proof}
    We define the set $\mathbb{M}_{2} := \{i : \|\bm{k}_i^\top - \bm{\kappa}_i^\top\bm{\mathcal{Q}}\| \ge 1/\sqrt{2} \} \supset \mathbb{M}$, which includes all misclassified nodes according to~\Cref{T:M2misclassified_nodes}.
    From~\Cref{T:KmatrixVsV,T:Davis-Kahan}, and~\Cref{T:main_result}, it follows that
    \begin{align*}
        q &= \frac{|\mathbb{M}|}{n} \le \frac{|\mathbb{M}_2|}{n} = \frac{1}{n}\sum_{i \in \mathbb{M}_2} 1 \le \frac{2}{n}\sum_{i \in \mathbb{M}_2} \|\bm{k}_i^\top - \bm{\kappa}_i^\top\bm{\mathcal{Q}}\|^2  =   \frac{2}{n}\|\bm K - \widetilde{\bm V}\bm{\mathcal{Q}}\|_F^2 \\
          & \le \frac{8}{n}\| \widetilde{\bm W} -\widetilde{\bm V} \bm{\mathcal{Q}}\|_F^2   \le \frac{8}{n} \left(\frac{\sqrt{8k}\|\sCov - \bm S(T)\|}{\tau\xi_k}\right)^2. 
        \end{align*}
\end{proof}

\subsection{Implications for parameter estimation}\label{sec:implications_inference_parameters}
Here we show how our concentration results allow us to prove \Cref{T:parameter_estimation_result}.

\begin{proof}
   Note that we can write the relative estimation error as:
   \begin{align*}
       \frac{|\hat{a} -a |}{a} &= \frac{|\lambda_2(\sCov)^{1/2T} \rho n  - \lambda_2(\bm S(T))^{1/2T} \rho n| }{\lambda_2(\bm S(T))^{1/2T} \rho n + \rho n} = \frac{|\lambda_2(\sCov)^{1/2T}  - \lambda_2(\bm S(T))^{1/2T}|}{\lambda_2(\bm S(T))^{1/2T} + 1} \\
                               &\le \frac{|\lambda_2(\sCov)  - \lambda_2(\bm S(T))|^{1/2T}}{\lambda_2(\nEA) + 1} \le \frac{\|\sCov - \bm S(T)\|^{1/2T}}{\lambda_2(\nEA) + 1}. \\
   \end{align*}
   The first inequality follows from Jensen's inequality since the function $f(x) = x^{1/2T}$ is concave. The second inequality is a direct application of Weyl's inequality.
   Hence, using \Cref{T:main_result} it follows that with probability at least $1-5s^{-1} - 2\epsilon$,
   \begin{equation*}
       \eta \le \frac{(2TM +B(T) )^{1/2T}}{\lambda_2(\nEA) + 1}.
   \end{equation*}
\end{proof}

\section{Numerical experiments}\label{sec:numerical_exp}
In this section we present some computational experiments with our algorithms for partition recovery and parameter estimation, to demonstrate their performance in practice.

\subsection{Partition recovery}

\begin{figure*}[tb!]
    \centering
    \includegraphics[]{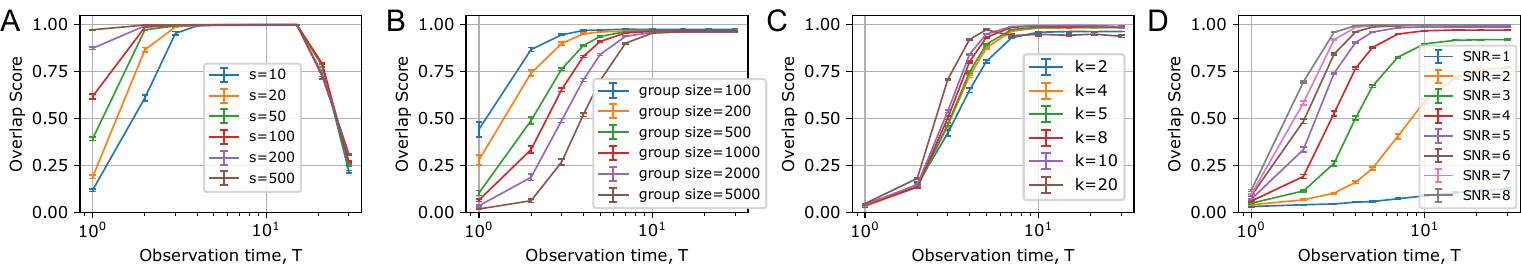}
    \caption{\textbf{Partition recovery performance for SBM inference}.
    \textbf{A-D.}
    The recovery performance of~\Cref{alg} is plotted against the observation time at which the samples are drawn.
    Each point in the plot corresponds to an average over 30 experiments, and the error bars denote the standard error of the mean. 
    Within~\Cref{alg} we run k-means $10$ times and pick the best result in terms of the k-means objective function as our final partition.
    \textbf{A.} Recovery performance for different numbers of samples, for the SBM discussed in the text.
    \textbf{B.} Recovery performance in a two-group planted partition model (see~\Cref{ex:example1}) for different group sizes ($\text{SNR}=5,\ s=50,\ \rho=30/n$).
    \textbf{C.} Recovery performance in a $k$-group planted partition model (see~\Cref{ex:example1}) for different numbers of groups ($n=5000, \ \text{SNR}=5,\ s=50,\ \rho=30/n$).
    \textbf{D.} Recovery performance in a $k=5$ group planted partition model (see~\Cref{ex:example1}) for different SNRs ($n=2000,\ s=50,\ \rho=30/n$).
    }
    \label{fig:part_rec1}
\end{figure*}
To gain intuition about the qualitative features of the partition recovery problem, we first consider the following example SBM consisting of $5$ groups with varying sizes $\bm n_g = [200,300,200,300,280]^\top$, and affinity matrix
\begin{equation*}\footnotesize
    \bm \Omega = 
    \begin{bmatrix}
        0.5& 0.3& 0.2& 0.1& 0.2\\ 
        0.3& 0.5& 0.2& 0.2& 0.2\\
        0.2& 0.2& 0.4& 0.2& 0.1\\
        0.1& 0.2& 0.2& 0.4& 0.2\\
        0.2& 0.2& 0.1& 0.2& 0.4\\
    \end{bmatrix}.
\end{equation*}
The partition recovery performance is  shown in~\Cref{fig:part_rec1}A for a varying number of samples $s$, as a function of the observation time $T$.
Here, the recovery performance is measured via the overlap score $Z$ defined as:
\begin{equation}\label{eq:overlap}
    Z = \frac{z_\text{actual} - z_\text{chance}}{1-z_\text{chance}},
\end{equation}
where  $z_\text{chance}$ is the probability of guessing the group assignments correctly and $z_\text{actual}$ is the fraction of nodes that are correctly classified.
The guessing probability is computed as $z_\text{chance} = max(\bm n_g)/n$, where $\max(\bm n_g)$ is the size of the largest group. 
Note that $z_\text{chance}$ is simply $\frac{1}{\#\text{groups}}$ for groups of equal size.
The actual overlap $z_\text{actual}$ is computed by solving a linear $k\times k$ optimal assignment problem between the inferred group labels and the planted partition.

Overall, we can infer the partitions remarkably well using our simple spectral inference strategy, even though we never observe the actual network.
The results in \Cref{fig:part_rec1}A display a couple of noteworthy features consistent with our technical results, which are worth examining.
Specifically, we see that there is a pronounced effect of the observation time: both for small as well as very large $T$, the performance is markedly worse than for intermediate times.
Thus, if we have a fixed budget for the number of samples, choosing the `right' observation time can improve recovery performance.
Obtaining more samples, in contrast, has a more moderate effect and is most beneficial for small times.

At an intuitive level, these results can be explained by viewing the network's effects as a linear filter.
On the one hand, for short observation times, the original white noise input is only filtered weakly and thus contains still a large noise component, which is detrimental to recovery performance.
On the other hand, as the output approaches its stationary state, the information that the signal carries about the network structure is completely washed out.
For intermediate times, however, there is a trade-off between attenuating the noise within the observed signal, and not losing too much information about the partition structure.
\Cref{T:main_result,T:result_misclassification} quantify precisely how this trade-off is governed by the spectral features of the network.

Using a planted partition model with an affinity matrix of the form $\bm \Omega = \frac{(a-b)}{n}\bm I_k + \frac{b}{n}\bm 1^{}_k\bm 1_k^\top$ (see~\Cref{ex:example1}), we further investigate 
(see~\Cref{fig:part_rec1}B-D). numerically how the partition recovery performance is affected by the group sizes, the number of groups, and the signal-to-noise ratio (SNR) defined by~\cite{Abbe2018},
\begin{equation}
    \text{SNR} = \frac{(a-b)^2}{ka + k(k-1)b}.
\end{equation}
We note that that $\text{SNR}=1$ corresponds to the well-known detectability limit in the sparse regime of the SBM, for asymptotically large networks, with $n\rightarrow \infty$ (see~\cite{Abbe2018} for a detailed review).

\Cref{fig:part_rec1}B shows that for a fixed number of groups (here $k=2$), the group sizes have a noticeable effect on recovery performance. On the other hand, if we fix the number of nodes, changing the number of groups does not have a noticeable effect. 
This hints at the fact that the number of nodes $n$ (and thus the number of nodes per group) is one of the key drivers of the difficulty of partition recovery.
Finally, in~\Cref{fig:part_rec1}D we plot the influence of the SNR on the detection performance.
As expected, a smaller SNR results in a dramatically worse  performance.
Remarkably, for $\text{SNR}=4$, using only $s=50$ samples leads to an almost perfect partition recovery,  for a suitable observation time.

\begin{figure*}[tb!]
    \centering
    \includegraphics[]{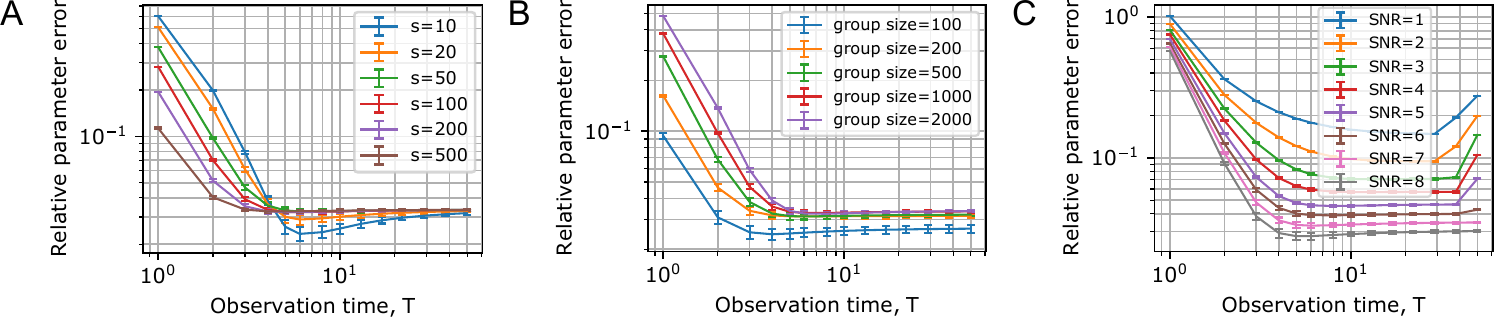}
    \caption{\textbf{Parameter recovery performance for SBM inference with an eigenvalue-based estimation strategy}.
    \textbf{A-C.}
    The recovery performance of~\Cref{alg2} in terms of the relative error $\eta = |a - \hat{a}| / a$ is plotted against the observation time at which the samples are drawn.
    Each point corresponds to an average over 30 experiments and the error bars denote the standard error of the mean. 
    \textbf{A.} Parameter recovery performance for varying number of samples for the SBM with $k=2$ groups as discussed in the text ($n=2000, \text{SNR}=7,\rho=30/n$).
    \textbf{B.} Parameter recovery performance in a $k=2$ group planted partition model (see~\Cref{ex:example1}) for varying group sizes ($\text{SNR}=7,s=50,\rho=30/n$).
    \textbf{C.} Parameter recovery performance in a $k=2$ group planted partition model (see~\Cref{ex:example1}) for different SNR ($n=2000,s=50,\rho=30/n$).}
    \label{fig:parameter_recovery}
\end{figure*}

\subsection{Model parameter estimation}
We now demonstrate numerically that we can infer the model parameters of the SBM using only the observed outputs.
For simplicity, we focus on the case of the planted partition model with $k=2$ groups of equal size, and assume that the average expected network density $\rho$, or equivalently the average expected mean degree $d_\text{av} = \rho n$, is known a priori (see also the discussion in the context of~\Cref{ex:example1} and~\Cref{sec:th_param_est}).

Since the estimation error is not directly related to 
 the partition recovery performance, let us first focus on the eigenvalue-based parameter estimation strategy, using the estimator $\hat a$ for the parameter $a$ given by \Cref{eq:estimator_lambda} (see~\Cref{sec:th_param_est}). 

\Cref{fig:parameter_recovery} shows the results of applying this parameter estimation scheme in numerical experiments. 
To quantify the parameter recovery performance we compute the relative recovery error $\eta = |a - \hat{a}| / a$. 
Similar to partition recovery, we see in~\Cref{fig:parameter_recovery}A that the observation time has a strong effect on the recovery performance and we also see a marked decrease in the error from more than $10\%$ initially, to around $3\%$ for longer observation times.

\begin{figure*}[tb!]
    \centering
    \includegraphics[]{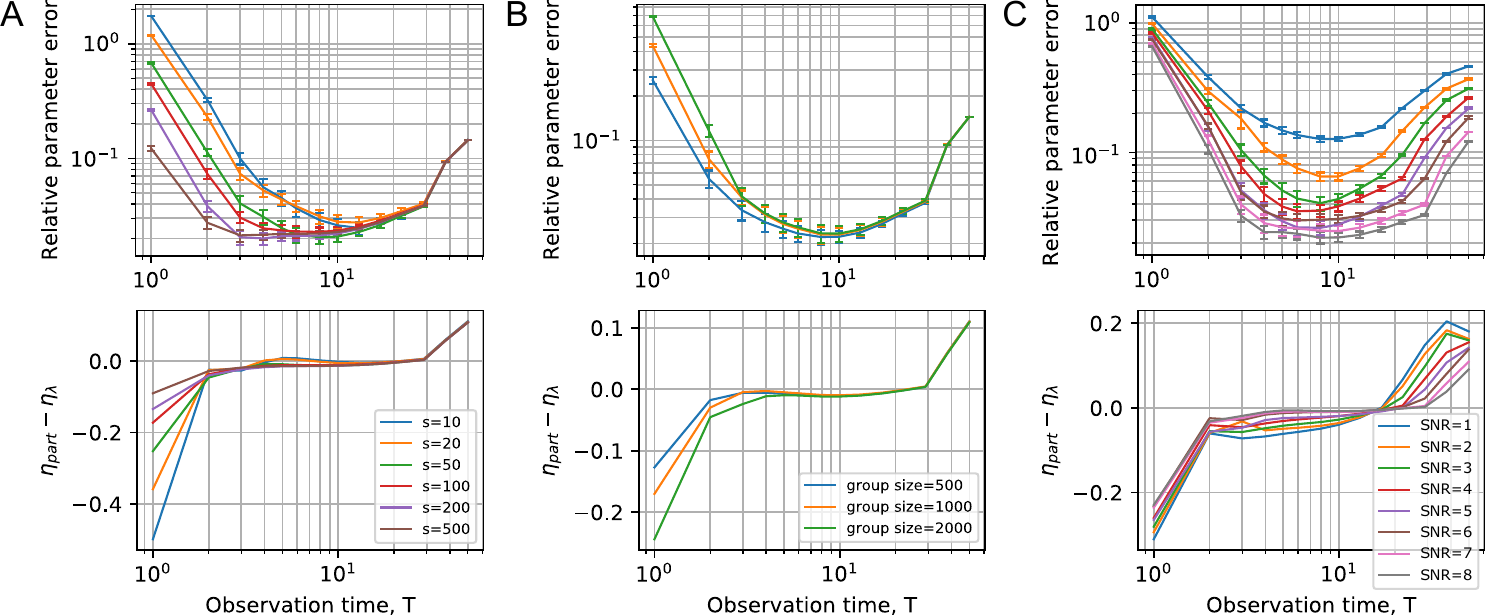}
    \caption{\textbf{Comparison of parameter recovery performance with different parameter estimation strategies}.
    \textbf{A-C.}
    The recovery performance of~\Cref{alg2} in terms of the relative error $\eta = |a - \hat{a}| / a$ is plotted against the observation time at which the samples are drawn.
    Top row: relative parameter estimation error $\eta_\text{part}$ of the estimation procedure based on first inferring a partition of the underlying SBM.
    Each point corresponds to an average over 30 experiments and the error bars denote the standard error of the mean.
    Bottom row: difference between relative parameter estimation error $\eta_\text{part}$ of the estimation procedure based on first inferring a partition from the data and the relative error $\eta_\lambda$ of the eigenvalue-based parameter inference strategy (see text). 
    \textbf{A.}~Parameter recovery performance for varying number of samples for the SBM with $k=2$ groups as discussed in the text ($n=4000, \text{SNR}=7,\rho=30/n$).
    \textbf{B.}~Parameter recovery performance in a $k=2$ group planted partition model (see~\Cref{ex:example1}) for varying group sizes ($\text{SNR}=7,s=50,\rho=30/n$).
    \textbf{C.}~Parameter recovery performance in a $k=2$ group planted partition model (see~\Cref{ex:example1}) for different SNR ($n=4000,s=50,\rho=30/n$).}
    \label{fig:parameter_recovery2}
\end{figure*}

\Cref{fig:parameter_recovery}B-C demonstrates the effect of the group size and the SNR on recovery performance.
Similar to the partition recovery problem, we observe that smaller group sizes (a smaller network) and a larger SNR lead to better performance.

Let us now consider a different parameter estimation strategy based on first recovering the groups in the SBM and then trying to recover the parameters based on the entries of the sample covariance matrix (see discussion in previous sections).
We concentrate again on the planted partition model with two equally-sized groups and a given density $\rho = (a+b)/(2n)$ as above.

Specifically, we estimate the parameter using the following approach.
First, we estimate the group labels $g_i$ of the nodes using~\Cref{alg} and sort the indices of the sample covariance $\sCov$ accordingly.
Next, we take the $t$-th root of $\sCov$ (via an SVD), and then compute the average magnitude of all off-diagonal terms that relate different groups, to obtain the following statistic:
\begin{equation}
    x = \frac{1}{|\{(i,j): g_i \neq g_j\}|}\sum_{(i,j): g_i \neq g_j} [\sCov^{1/T}]_{ij}.
\end{equation}

From~\Cref{E:example_entries}) we know that 
\begin{equation}\label{E:cov_entries}
[\sCov^{1/T}]_{ij} \approx 
\begin{cases}
\frac{1}{n}\left(\frac{a-b}{a+b}\right)^{2} + \frac{1}{n} & \text{if } g_i = g_j,\\
-\frac{1}{n}\left(\frac{a-b}{a+b}\right)^{2} + \frac{1}{n} & \text{if } g_i \neq g_j,
\end{cases}
\end{equation}
which implies that $x \approx -\frac{1}{n}\left(\frac{a-b}{a+b}\right)^{2} + \frac{1}{n}$.
We then use the following estimator for the parameter $a$:
\begin{equation}
    \hat a_\text{part} = \rho n\sqrt{1-nx} + \rho n
\end{equation}

Figure~\Cref{fig:parameter_recovery2} displays the performance of this estimator in terms of the error $\eta_\text{part} = |\hat{a}_\text{part} - a|/a$, in comparison to the relative error $\eta_\lambda= |\hat{a} - a|/a$ of the eigenvalue-based estimation strategy.

\subsection{Real-world networks}
In this section, we examine the performance of our methods on real-world networks that are not known to be generated from an SBM.
In this final set of numerical experiments, we test two real-world networks, on which we assume a diffusion dynamics of the form~\eqref{eq:observation_model_simple} takes place.

The first network we study is the well-known Karate-club network which was initially charted by Zachary~\cite{Zachary1977}, describing the social interactions among 34 individuals of a Karate-club that eventually split into two factions.
We use this split as our `ground-truth' partition to compare against.
Although there is \emph{a priori} no reason for why the split should be correlated with the graph structure~\cite{Peel2017}, this finding is now well-supported in the literature.

The results of applying our procedure for a range of different sampling times and number of samples available are shown in~\Cref{fig:real_world}A.
Interestingly, we find that already with three samples, we can achieve a recovery performance that is as good as applying a corresponding spectral clustering procedure to the full network --- which in this case corresponds to making one classification error with respect to the reported fission of the club.

As a second example, we study a social interaction network consisting of the friendship relationships between students at Caltech, a dataset that was originally studied by Traud et al.~\cite{Traud2011,Traud2012} in a comparative study of a larger set of such networks.
While there is again no clear `ground truth' partition of the network, some additional node covariates such as gender or the major of the students are examined as part of their study~\cite{Traud2011,Traud2012}.
Traud et al.~\cite{Traud2011,Traud2012} observed that in the context of the social network of Caltech students the node covariate `house membership' (dormitories) correlates strongly with a putative community structure in the network, even though the house-membership cannot explain the structure of the network completely.

Motivated by these finding, we study the partition recovery performance with respect to the partition of the nodes induced by the house-membership data in~\Cref{fig:real_world}B-C.
Here we limit ourselves to studying the largest connected component of the network, where we remove all the nodes for which no house-membership label is available.
This results in a network of $594$ nodes split between 8 houses (groups).
As the structure of the network is not completely explained by the house-membership data, we recover the induced partition only with an overlap score of $0.65$ (see \Cref{fig:real_world}B).
We also show our recovery performance relative to the partition we would obtain from the spectral clustering procedure that uses $\bm L$ (i.e., knowledge of the complete network) instead of $\sCov$ for clustering.
As can be seen in~\Cref{fig:real_world}C, our recovered partition is almost perfectly correlated with the result of the clustering-based procedure, highlighting that our method  indeed performs as expected and that the low score with respect to the house-membership partition is largely due to a lack of a correlation between the house membership variable and the network block structure.

\begin{figure*}[tb!]
    \centering
    \includegraphics[]{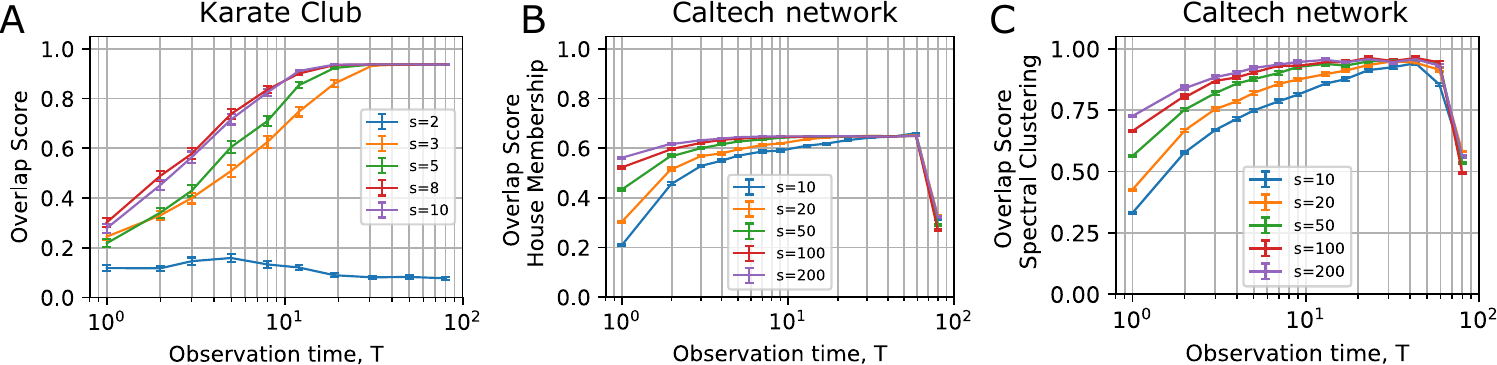}
    \caption{\textbf{Recovering partition structure from dynamics on real-world networks}.
    Each point corresponds to an average over 200 experiments. The error bars denote the standard error of the mean. 
    \textbf{A.}~Partition recovery performance of our method on the Karate club network with respect to the partition induced by the final split of the Karate club, as reported by Zachary~\cite{Zachary1977} (see text).
    \textbf{B.}~Partition recovery performance of our method on the social network of Caltech students, with respect to the partition induced by the house membership of the students, as reported by Traud et al.~\cite{Traud2011,Traud2012}.
    \textbf{C.}~Partition recovery performance of our method on the social network of Caltech students, with respect to the partition obtained from spectral clustering on the fully observed network.}
    \label{fig:real_world}
\end{figure*}

\section{Discussion}\label{sec:discussion}
Graph-based tools have become prevalent for the analysis of a range of different systems across the sciences. 
Network inference, the problem of determining the topology of a network based on a set of nodal observations, has accordingly become a paramount tool to enable this type of analysis.
In this work, we have advocated a fresh look onto this problem for applications in which exact inference is impossible, e.g., due to sampling constraints.
Instead of inferring the exact network in terms of its adjacency matrix, we infer a generative model of the network structure, thereby learning a probability distribution over networks.
Importantly, for a low-dimensional statistical model such as the SBM discussed here, this procedure can be implemented using far fewer samples than would be required to infer the exact network structure.

The proposed method is especially relevant for scenarios in which we aim to analyze the network further from a statistical perspective, e.g., if our ultimate goal is to perform community detection on the network.
An added benefit of having fitted a generative model is that we can employ the model for further tasks afterwards, e.g., to generate surrogate network data.

Most of our algorithmic ideas can be extended to more general functions $f$ which inherit the low-rank structure of the (expected) SBM, provided an appropriate concentration result for the system matrix under consideration can be derived.
For instance, this includes those functions describing a low-pass graph filter, as considered in~\cite{Wai2018,schaub_2019_spectral}, but also functions of the type $f(\bm A) = \sum_\ell \alpha_\ell \bm A^\ell$ for some suitably chosen parameters $\alpha_\ell$.

To adjust the proposed algorithms to such setups, we would have to establish and prove appropriate relationships between the spectral properties of the system matrix on one side and the partition structure and parameters of the model on the other.
Specifically, step 3 of~\Cref{alg} uses the particular structure of the eigenvectors of $\bm L$ (see~\Cref{T:eigenvectors_and_partitions}).
Of course, the set of equations relating the model parameters to the statistics of the observed outputs will in general be different and will depend on the specific model chosen.
In fact, some of the non-identifiability problems we encountered here in terms of recovering the model parameters may not arise for other classes of system matrices $f(\bm A)$.

Our work also opens up several additional avenues for future research. 
This includes the analysis of the inference of more realistic network models based on a setup as discussed here, ranging from degree-corrected SBMs~\cite{Karrer2011} all the way to graphons~\cite{Lovasz2012,Avella-Medina2017}.
Likewise, it will be interesting to extend the above results to more general dynamical processes, including nonlinear dynamics.
Another appealing line of future work would be to consider scenarios in which only a part of the network is observed~\cite{Mauroy2017}, or additional (control) inputs are possible.

\appendix
\section{Omitted proofs}
\subsection{Proof of~\Cref{T:eigenvectors_and_partitions}}
In order to prove~\Cref{T:eigenvectors_and_partitions}, we provide a slightly more general discussion about the relationship between the eigenvectors of $\nEA$ and the partition indicator matrix $\bm G$.
\Cref{T:eigenvectors_and_partitions} then follows directly from~\Cref{T:eigenvectors2} below.

\begin{definition}[Normalized affinity matrix]
    Consider an SBM with affinity matrix $\bm \Omega$ and partition indicator matrix $\bm G$.
    Let us denote the vector of group sizes by $\bm n_g = \bm G^\top \bm 1$ and define $\bm N_g := \diag(\bm n_g) = \bm G^\top \bm G$.
    We then define the normalized affinity matrix as
    \begin{equation*}
        \bm \Gamma = \bm N_g^{1/2}\diag(\bm \Omega\bm n_g )^{-1/2}\bm \Omega \diag(\bm \Omega\bm n_g )^{-1/2}\bm N_g^{1/2}.
    \end{equation*}
\end{definition}

Note that $\nEA$ can be written as
\begin{equation}
    \nEA = \bm G \bm N_g^{-1/2}\bm \Gamma \bm N_g^{-1/2}\bm G^\top.
\end{equation}

\begin{lemma}
     Let $\nEA$ be the expected normalized adjacency matrix of an SBM with a corresponding normalized affinity matrix $\bm \Gamma$ with eigenvalues $\lambda_i$ and corresponding eigenvectors $\bm u_i$. 
     Then, any eigenvector $\bm u_i$ 
     of $\bm \Gamma$ gives rise to an eigenvector $\bm v_i = \bm G \bm N_g^{-1/2}\bm u_i$ of $\nEA$ with the same eigenvalue.
\end{lemma}
\begin{proof}
   A simple computation shows that 
   \begin{align*}
       \nEA \bm G \bm N_g^{-1/2}\bm u_i &= \bm G \bm N_g^{-1/2}\bm \Gamma \bm N_g^{-1/2}\bm G^\top\bm G \bm N_g^{-1/2}\bm u_i = \bm G \bm N_g^{-1/2}\bm \Gamma\bm u_i = \lambda_i \bm G \bm N_g^{-1/2}\bm u_i = \lambda_i\bm v_i.
   \end{align*}
\end{proof}

\begin{corollary}\label{T:eigenvectors1}
    The partition indicator matrix $\bm G$ can be written as a linear combination of the eigenvectors $\bm v_1, \ldots, \bm v_k$ corresponding to the $k$ nonzero eigenvectors of $\nEA$.
\end{corollary}
\begin{proof}
    From the previous lemma we know that the eigenvectors of $\nEA$ can be written as  $\bm v_i = \bm G \bm N_g^{-1/2}\bm u_i$, where $\bm u_i$ are the eigenvectors of $\bm \Gamma$.
    Assembling these relationships into matrices we obtain
    \begin{equation}\label{E:proof_lem_050}
        \bm V = \bm G \bm N_g^{-1/2} \bm U.
    \end{equation}
    Since $\bm U$ is an orthogonal matrix and $\bm N_g^{1/2}$ is diagonal, both matrices are invertible. It then follows that
\begin{equation}
    \bm G = \bm V \bm U^\top \bm N_g^{1/2}.
\end{equation}

\end{proof}

\begin{corollary}\label{T:eigenvectors2}
    Consider the row-normalized eigenvector matrix $\widetilde{\bm V} = \diag(\bm V\bm V^\top)^{-1/2}\bm V$.
    There exists an orthogonal transformation $\bm{\mathcal U}$ such that $\widetilde{\bm V} = \bm G \bm{ \mathcal U}$.
\end{corollary}
\begin{proof}
    Note that 
    \begin{align*}
        \diag(\bm V\bm V^\top) &= \diag(\bm G \bm N_g^{-1/2} \bm U  \bm U^\top\bm N_g^{-1/2}\bm G^\top)= \diag(\bm G\bm N_g^{-1}\bm G^\top).
    \end{align*}
    Combining this equation with \eqref{E:proof_lem_050} and the fact that $\bm G$ is an indicator matrix it follows that
    \begin{align*}
        \widetilde{\bm V} & = \diag(\bm G \bm N_g^{-1} \bm G^\top)^{-1/2}\bm V = \diag(\bm G \bm N_g^{-1} \bm G^\top)^{-1/2} \bm G \bm N_g^{-1/2} \bm U  = \bm G \bm N_g^{1/2} \bm N_g^{-1/2}\bm U = \bm G \bm U.
    \end{align*}
    Since $\bm U$ corresponds to the eigenvector basis of the symmetric normalized affinity matrix $\bm \Gamma$, we know that $\bm U$ is orthogonal.
\end{proof}

\subsection{Proof of~\Cref{T:concentration_A}}
\begin{proof}
    We define the matrices $\ED := \mathbb E[\bm D]$ and ${\bm H := \ED^{-1/2}\bm A \ED^{-1/2}}$ for notational convenience.
    From the triangle inequality we have
    \begin{equation}\label{E:proof_lem_010}
        \|\nA - \nEA\| \leq \|\nA - \bm H\| + \|\bm H - \nEA\|.
    \end{equation}
    To obtain our desired concentration inequality, we bound the two terms on the right-hand side of~\eqref{E:proof_lem_010} separately.

    For the first term we note that since $\|\nA\| \le 1$,
    \begin{flalign}\label{E:proof_lem_020}
        \|\nA - \bm H\| &= \| \bm D^{-1/2}\bm A \bm D^{-1/2} -\ED^{-1/2}\bm A\ED^{-1/2}\| =\| \nA -\ED^{-1/2}\bm D^{1/2}\nA\bm D^{1/2}\ED^{-1/2}\| \nonumber\\
                        &=\left \| (\bm I -\ED^{-1/2}\bm D^{1/2})\nA\bm D^{1/2}\ED^{-1/2} + \nA(\bm I -\bm D^{1/2}\ED^{-1/2})\right \| \nonumber\\
                        &\le \|\bm I -\ED^{-1/2}\bm D^{1/2}\| \| \bm D^{1/2}\ED^{-1/2}\| +\|\bm I -\ED^{-1/2}\bm D^{1/2}\|.
    \end{flalign}

    We can now analyze the term $\|\bm I -\ED^{-1/2}\bm D^{1/2}\|$ as follows:
    \begin{align*}
        \|\bm I -\ED^{-1/2}\bm D^{1/2}\| = \max_i \left|1 - \sqrt{D_{ii} /\mathcal D_{ii}} \right | \le \max_i \left |1 - D_{ii}/\mathcal D_{ii}\right|,
    \end{align*}
    where we have used the fact that the spectral norm of a diagonal matrix is its maximal diagonal entry.
    To make use of the above result, we need to obtain a concentration result for the degrees. 
    Recalling that ${\delta}_i = \mathbb{E}(d_i) = \mathcal{D}_{ii}$, we can use Chernoff's inequality to obtain a bound of the form:
    \begin{align*}
        \mathbb{P}(|d_i - \delta_i| \ge t \delta_i ) \le \frac{\epsilon}{2n} \quad \text{for } t \ge \sqrt{\frac{\ln(4n/\epsilon)}{\delta_i}}.
    \end{align*}

    If we choose $t_0=\sqrt{\ln(4n/\epsilon)/\delta_\text{min}}$, the above bound holds for all $i$, and by a simple transformation we obtain $\mathbb{P}(|d_i/\delta_i - 1| \ge t_0) \le \epsilon /2n$ for all $i$.
    Using the union bound, we thus have with probability at least $1-\epsilon/2$:
    \begin{align*}
        \|\bm I -\ED^{-1/2}\bm D^{1/2}\| = \max_i \left|1 - \sqrt{\frac{D_{ii}}{ \mathcal D_{ii}}} \right | \le \sqrt{\frac{\ln(4n/\epsilon)}{\delta_\text{min}}}.
    \end{align*}
    Therefore, plugging this result into~\eqref{E:proof_lem_020}, with probability $1-\epsilon/2$ it holds that
    \begin{align}\label{E:proof_lem_030}
        \|\nA - \bm H\| \le t_0(t_0+1) + t_0.
    \end{align}

    For the second term $\|\bm H - \nEA\|$, we can use the following matrix concentration inequality.
    \begin{lemma}[Bernstein Matrix inequality~\cite{chung2011spectra}]\label{L:bernstein_matrix}
        Let $\bm X_{i}$ be independent $n\times n$ dimensional random Hermitian matrices such that ${\|\bm X_{i} - \mathbb{E}[\bm X_{i}]\|\le C \quad \forall i}$, and  $\bm X = \sum_{i} \bm X_{i}$.  
        Then for any $a>0$:
    \begin{equation*}
    \mathbb{P}(\|\bm X - \mathbb{E}[\bm X]\| > a) \le 2n \exp\left(\frac{-a^2}{2v^2 + 2Ca/3}\right),
    \end{equation*}
    where $v^2 = \| \sum_{i} \mathrm{var}(\bm X_{i})\| =\|\sum_i \mathbb{E}[\bm X_i^2] - \mathbb{E}[\bm X_i]^2 \| $.
    \end{lemma}

    To apply this result, note that we can decompose ${[\bm H - \nEA]}= \sum_{i\le j}\bm X_{ij}$, where $\bm X_{ij}$ are the Hermitian matrices:
    \begin{equation*}
        \bm X_{ij}= \begin{cases}
            \frac{A_{ij}-p_{ij}}{\sqrt{\delta_i\delta_j}}(\bm E^{ij} + \bm E^{ij}), & \text{for } i\neq j,\\
            \frac{A_{ii}-p_{ii}}{\delta_i}\bm E^{ii}, & \text{for } i = j.
        \end{cases}
    \end{equation*}
    Here, $\bm E^{ij}$ is the canonical basis matrix with entries $[\bm E^{ij}]_{ij}=1$ and $0$ otherwise.
    As $\mathbb{E}[\bm X_{ij}] = 0$ we compute $\text{var}(\bm X_{ij}) = \mathbb{E}[\bm X_{ij}^2]$ as follows:
    \begin{equation*}
        \text{var}(\bm X_{ij}) = \begin{cases}{}
            \frac{1}{\delta_i\delta_j}(p_{ij}-p_{ij}^2) (\bm E^{ii} + \bm E^{jj}), & \text{for } i \neq j,\\
            \frac{1}{\delta_i^2}(p_{ii}-p_{ii}^2) \bm E^{ii}, & \text{for } i = j.
        \end{cases} 
    \end{equation*}
    From this we can bound $v^2$ as 
    \begin{align*}
        v^2 &= \left\| \sum_{i=1}^n \sum_{j=1}^n \frac{1}{\delta_i\delta_j} (p_{ij}- p_{ij}^2)\bm{E}^{ii}\right\| = \max_i \left ( \sum_{j=1}^n \frac{1}{\delta_i\delta_j}p_{ij} -  \sum_{j=1}^n \frac{1}{\delta_i\delta_j}p_{ij}^2\right )\\
            &\le \max_i \left ( \frac{1}{\delta_\text{min}}\sum_{j=1}^n \frac{1}{\delta_i}p_{ij}\right ) = \frac{1}{\delta_\text{min}},
    \end{align*}
    where we have used the bound $1/\delta_\text{min} \ge 1/\delta_j$.

    Based on the above calculations, we select ${a= \sqrt{3\ln(4n/\epsilon)/\delta_\text{min}} = \sqrt{3}t_0}$, which leads to our desired bound for the second term
    \begin{align}\label{E:proof_lem_040}
    \mathbb{P}(\|\bm H - \nEA\| > a) &\le 2n \exp\left(\frac{-a^2}{2/\delta_\text{min} + 2a/(3\delta_\text{min})}\right) \nonumber\\
                                                         &\le 2n \exp\left( -\frac{3\ln(4n/\epsilon)}{3}\right) \le \epsilon/2,
    \end{align}
where the first inequality follows from the fact that $C \leq 1/\delta_{\min}$ (cf.~Lemma~\ref{L:bernstein_matrix}), and for the second inequality we used the fact that $\mathcal M(\bm{\mathcal A}) \in \mathcal M_n(\epsilon)$ to determine a lower bound on $\delta_{\min}$.
    Finally, by combining the results in \eqref{E:proof_lem_030} and \eqref{E:proof_lem_040}  we obtain that with probability $1-\epsilon$,
    \begin{align*}
        \|\nA - \nEA\|  \le \|\nA - \bm H\| + \|\bm H - \nEA\| \le a + t_0^2 +2t_0 \le 3a.
    \end{align*}

\end{proof}

\bibliographystyle{siamplain}
\bibliography{references}

\end{document}